\documentclass{article}

\usepackage[preprint]{neurips_2026}

\usepackage{graphicx}
\usepackage{algorithm}
\usepackage{algorithmic}

\usepackage{multirow}
\usepackage{enumitem}
\usepackage{caption}

\usepackage[utf8]{inputenc} 
\usepackage[T1]{fontenc}    
\usepackage{hyperref}       
\usepackage{url}            
\usepackage{booktabs}       
\usepackage{amsfonts}       
\usepackage{nicefrac}       
\usepackage{microtype}      
\usepackage{xcolor}         
\usepackage{booktabs}
\usepackage{wrapfig}
\usepackage{placeins}
\usepackage{subcaption}
\usepackage{float}

\usepackage{amsmath,amssymb,amsthm,amsfonts,mathtools,bm,booktabs,xcolor,enumitem,hyperref}
\usepackage{natbib}

\newtheorem{assumption}{Assumption}
\newtheorem{theorem}{Theorem}

\newtheorem{lemma}{Lemma}
\newtheorem{corollary}{Corollary}
\newtheorem{remark}{Remark}

\newcommand{\E}{\mathbb{E}}
\newcommand{\R}{\mathbb{R}}

\newcommand{\tr}{\mathrm{tr}}

\newcommand{\Loss}{\mathcal{L}}
\newcommand{\Gauss}{\mathcal{N}}

\title{Adaptive Selection of LoRA Components in Privacy-Preserving Federated Learning}

\author{
    Myoungjun Kim \\
    Myongji University\\
    \texttt{tjmjtjmj2237@mju.ac.kr}
    \And
    Sangwoo Park\\
    King's College London\\
    \texttt{sangwoo.park@kcl.ac.uk}
    \And
    Yoseob Han \\
    Soongsil University\\
    \texttt{yoseob.han@ssu.ac.kr}
    \And
    Jin-Hyun Ahn \\
    Myongji University\\
    \texttt{wlsgus3396@mju.ac.kr}\\
}

\begin{document}
\maketitle

\begin{abstract}
Differentially private federated fine-tuning of large models with LoRA suffers from aggregation error caused by LoRA's multiplicative structure, which is further amplified by DP noise and degrades both stability and accuracy. Existing remedies apply a single update mode uniformly across all layers and all communication rounds (or alternate them on a fixed schedule), ignoring both the structural asymmetry between the two LoRA factors and the round-wise dynamics of training. We propose AS-LoRA, an adaptive framework defined by three axes \emph{(i)} layer-wise freedom, in which each layer independently selects its active component, $\emph{(ii)}$ round-wise adaptivity, in which the selection updates over communication rounds, and \emph{(iii)} a curvature-aware score derived from a second-order approximation of the loss. Theoretically, AS-LoRA eliminates the reconstruction-error floor of layer-tied schedules, accelerates convergence, implicitly biases solutions toward flatter minima, and incurs no additional privacy cost. Across GLUE, SQuAD, CIFAR-100, and Tiny-ImageNet under strict DP budgets and non-IID partitions, AS-LoRA improves over the federated LoRA baselines by up to $+7.5$ pp on GLUE and $+12.5$ pp on MNLI-mm for example, while matching or exceeding SVD-based aggregation methods at $33\text{--}180 \times$ lower aggregation cost and with negligible communication overhead. Code for the proposed method is available at \url{https://anonymous.4open.science/r/as_lora-F75F/}.
\end{abstract}

\vspace{-3mm}

\section{Introduction} 
Federated learning (FL) enables collaborative training over decentralized data without exposing raw samples \cite{mcmahan2017communication, bonawitz2019towards}, yet the communicated updates remain vulnerable to attacks such as Deep Leakage from Gradients \cite{zhu2019deep} and Membership Inference \cite{shokri2017membership}. Differential privacy (DP) has thus become a standard safeguard in this setting \cite{zheng2021federated, naseri2020local}. Meanwhile, as foundation models such as GPT-4 \cite{achiam2023gpt}, LLaMA \cite{touvron2023llama}, PaLM2 \cite{anil2023palm}, ViT \cite{dosovitskiy2020image}, and Swin-T \cite{liu2021swin} continue to scale, fine-tuning them on resource-constrained clients quickly becomes prohibitive. Low-Rank Adaptation (LoRA) \cite{hu2022lora}, which adapts a frozen weight $W_0$ via a small low-rank update $W=W_0 +BA$, has therefore become the predominant approach for Parameter-Efficient Fine-Tuning (PEFT) in FL with foundation models, incurring minimal performance degradation compared to full fine-tuning.

However, applying LoRA in private FL is non-trivial. Its multiplicative structure induces an \emph{aggregation error} across clients, which is further amplified into a \emph{quadratic noise term} under DP-SGD, and we defer the detailed mechanism to Section \ref{sec:back}. \emph{FedLoRA} \cite{zhang2024towards}, which is the most direct adaptation of LoRA to FL, suffers from both effects and degrades sharply under tight privacy budgets and non-IID partitions. Several remedies have since been proposed to alleviate those issues. \emph{FFA-LoRA} \cite{sun2024improving} freezes $A$ and trains only $B$, \emph{RoLoRA} \cite{chen2025robust} alternates updates between $A$ and $B$ across rounds, and \emph{FedSVD} \cite{lee2025fedsvd} also trains only $B$ while applying a server-side Singular Value Decomposition (SVD). Yet all of them adopt structurally simple update strategies that impose the \emph{same mode on every layer} and follow \emph{temporally static patterns}, which overlook both the structural asymmetry between $A$ and $B$, and the round-wise variation of their relative importance. This motivates an adaptive framework in which the active LoRA component is selected per layer and per round, as depicted in Figure \ref{fig:AS-LoRA}.


\begin{figure}[t]
    \centering
    \includegraphics[width=1.0\textwidth]{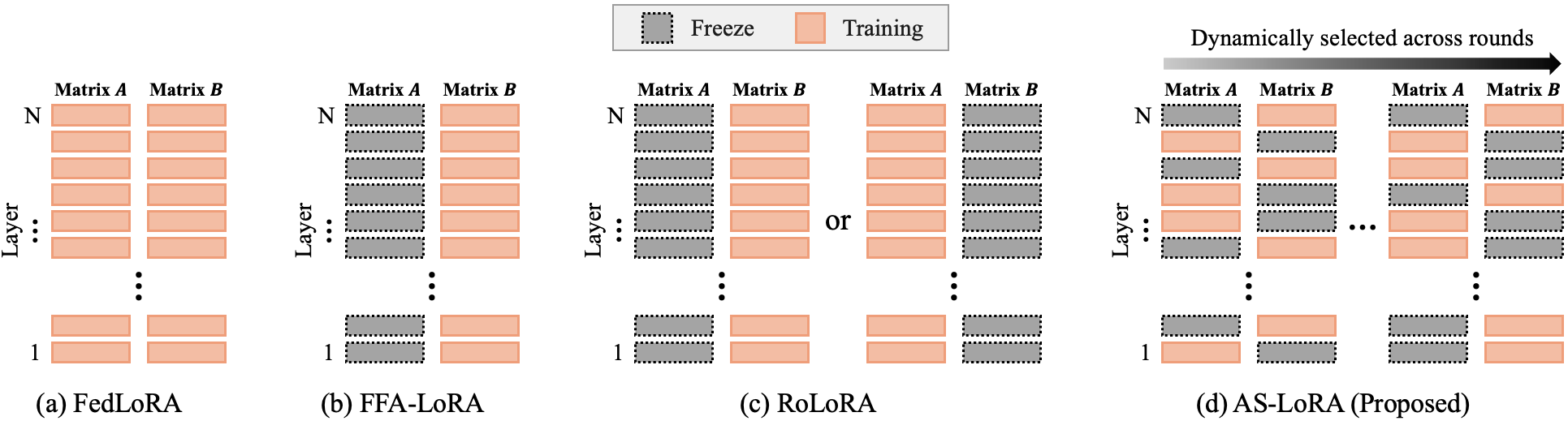}
    \caption{Comparison of update patterns in FedLoRA, FFA-LoRA, RoLoRA, and the proposed Adaptive Selection LoRA (AS-LoRA).}
    \label{fig:AS-LoRA}
    \vspace{-20pt}
\end{figure}

\begin{wraptable}{r}{0.46\linewidth}
\centering
\vspace{-13pt}
\scriptsize
\setlength{\tabcolsep}{5.5pt}
\renewcommand{\arraystretch}{1.15}
\caption{Performance comparison of fixed alternating LoRA optimization schedules on 3 GLUE tasks. Best results are shown in \textbf{bold} and second-best results are \underline{underlined}.}
\label{tab:fixed_schedule}
\begin{tabular}{lcccc}
\toprule
\textbf{Schedule} & \textbf{MNLI (m)} & \textbf{MNLI (mm)} & \textbf{SST-2} & \textbf{QQP} \\
\midrule
\texttt{BA} \cite{chen2025robust} & \textbf{81.17} & \underline{81.96} & 92.09 & \underline{84.45} \\
\texttt{BAA} & 80.26 & 81.20 & \textbf{92.89} & 84.21 \\
\texttt{BBA} & \underline{81.10} & \textbf{81.98} & 92.30 & \underline{84.45} \\
\texttt{BBAA} & 80.78 & 81.55 & \underline{92.66} & \textbf{84.49} \\
\bottomrule
\end{tabular}
\end{wraptable}

In fact, prior work~\cite{hayou2024lora+} shows that the two low-rank matrices contribute asymmetrically to the optimization process across training stages, owing to their different initialization schemes and structural characteristics. To examine whether this asymmetry can be handled by simple fixed schedules, we conduct preliminary experiments that compare several alternating patterns including \texttt{BA}, \texttt{BBA}, and \texttt{BAA}, as shown in Table~\ref{tab:fixed_schedule}. No single schedule attains the best performance across all tasks, which indicates that fixed and heuristic alternating schedules are inherently suboptimal, and that the choice of which low-rank component to optimize should instead be made adaptively according to current training dynamics. Consistent with this observation, the active LoRA component selected by our proposed AS-LoRA varies dynamically across both layers and communication rounds throughout training, as illustrated in Figure~\ref{fig:selected}, which presents a representative trajectory whose full details are deferred to Appendix~\ref{sec:Appendix D}.

Building on those observations, we propose AS-LoRA (Adaptive Selection LoRA), and the main contributions of this paper are summarized as follows.

\begin{wrapfigure}{r}{0.45\linewidth}
\centering
\vspace{-23pt}
\includegraphics[width=\linewidth]{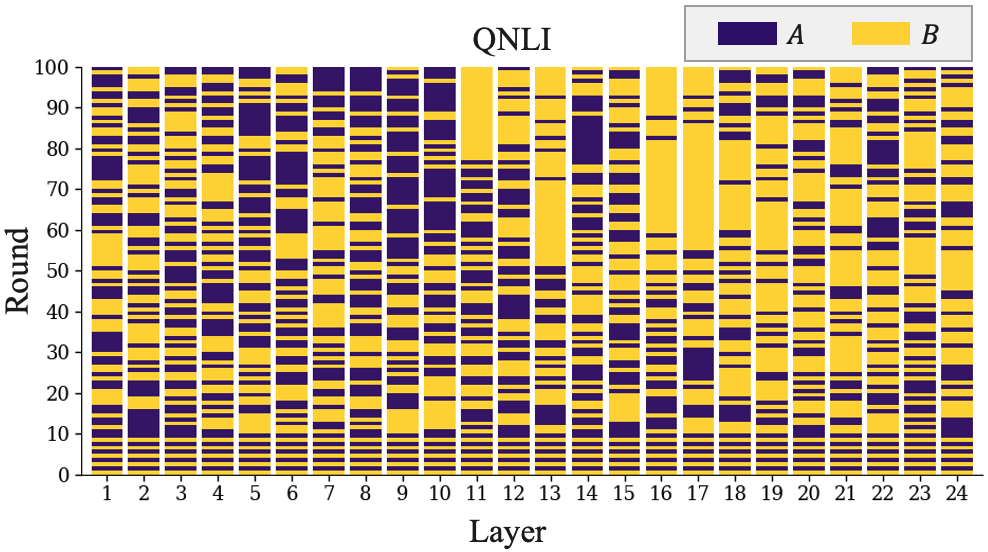}
\caption{Selection of LoRA components with AS-LoRA (ours) on QNLI across layers/rounds. Different layers favor different components, and the preferred component varies across rounds, supporting the need for adaptive selection.}
\label{fig:selected}
\vspace{-5pt}
\end{wrapfigure}
\textbf{Layer-wise adaptive LoRA optimization.}
To the best of our knowledge, this is the first work that extends alternating LoRA optimization to a adaptive selection framework. We further prove that this layer-wise freedom eliminates the irreducible reconstruction-error floor that any layer-tied schedule must incur.\\
\textbf{Round-wise adaptive mode selection.}
We introduce a temporally adaptive mechanism that updates the selected LoRA component across communication rounds based on EMA-smoothed historical scores, which allows the model to adjust its optimization strategy as training progresses.\\
\textbf{Curvature-aware scoring function.}
We design a scoring function derived from a second-order approximation of the loss, which captures curvature-aware component importance. We prove that this score-driven selection accelerates convergence with a quantifiable gain over fixed schedules and implicitly biases solutions toward flatter minima, which we further verify empirically through loss-landscape and sharpness analyses. To handle the channel-wise gradient outliers that are characteristic of LoRA~\cite{jung2025gralora}, we further apply a Gaussian random projection \cite{johnson1984extensions}, which yields more reliable score estimates. \\
\textbf{Privacy at no extra cost.}
We prove that the data-dependent component selection of AS-LoRA incurs no additional privacy cost beyond standard DP-SGD, which follows from the post-processing invariance of DP.

\vspace{-3mm}

\section{Background} 
\label{sec:back}

\textbf{Differential privacy \cite{dwork2014dpfoundations}.}  Formally, a randomized mechanism $\mathcal{M}$ is said to satisfy
$(\epsilon,\delta)$-differential privacy
if, for any two neighboring datasets $D$ and $D'$
that differ in exactly one data record,
and for any measurable subset $E$ of the output space,
the following condition holds:
$
\Pr[\mathcal{M}(D) \in E]
\leq e^{\epsilon} \Pr[\mathcal{M}(D') \in E] + \delta.
$ A common approach to achieving differential privacy in federated learning
is differentially private stochastic gradient descent (DP-SGD) \cite{abadi2016deep}.
Formally, let $g_i = \nabla_\theta \ell(\theta; x_i)$ denote the gradient
computed from a single training example $x_i$.
DP-SGD first applies per-sample gradient clipping. And the clipped gradients are then aggregated and perturbed with Gaussian noise:
$
\textit{(i)}\quad
g_i^{\mathrm{clip}} = g_i \cdot \min\!\left(1, \frac{C}{\|g_i\|_2}\right), \quad
\textit{(ii)}\quad
g^{\mathrm{noisy}} = \frac{1}{B}\left(\sum_{i=1}^{B} g_i^{\mathrm{clip}} + \mathcal{N}\!\left(0, \sigma^2 C^2 \mathbf{I}\right)\right),
$
where $C$ is a predefined clipping norm that bounds the sensitivity, $B$ is the batch size, $\sigma$ is the noise multiplier, and $\mathcal{N}(0, \sigma^2 C^2 \mathbf{I})$ denotes isotropic Gaussian noise.
Finally, the model parameters are updated using the noisy gradient:
$
\theta_{t+1} = \theta_t - \eta \, \bar{g},
$
where $\eta$ is the learning rate.

\textbf{Aggregation error in federated LoRA \cite{wang2024flora}.} Given a pre-trained weight matrix $W_0 \in \mathbb{R}^{d_{out} \times d_{in}}$,
LoRA parameterizes the adapted weight as $W = W_0 + BA$, where $B \in \mathbb{R}^{d_{out} \times r}$ and $A \in \mathbb{R}^{r \times d_{in}}$.
In FL, the server aggregates local LoRAs $\{(A_k, B_k)\}_{k=1}^K$ via averaging as
$
\bar{A} = \frac{1}{K}\sum_{k=1}^K A_k, \quad
\bar{B} = \frac{1}{K}\sum_{k=1}^K B_k.
$
However, due to the multiplicative structure of LoRA,
$
\frac{1}{K}\sum_{k=1}^K B_k A_k \;\neq\; \bar{B}\,\bar{A}.
$
This mismatch leads to an aggregation error that accumulates across communication rounds,
and becomes more pronounced under heterogeneous client data distributions.

\begin{wrapfigure}{r}{0.36\linewidth}
    \centering
    \vspace{-5mm}
    \includegraphics[width=\linewidth]{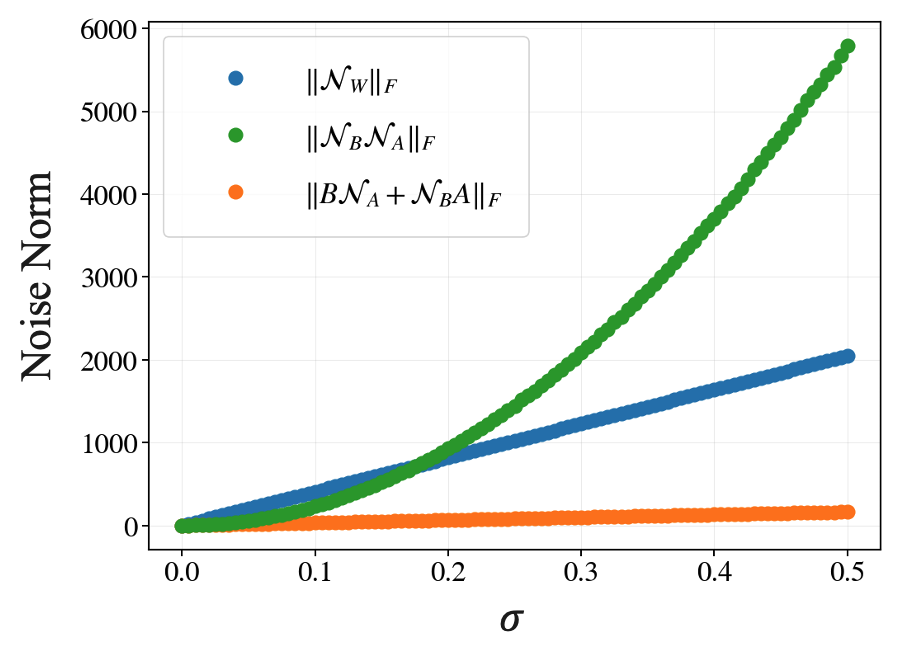}
    \caption{
    Scaling behavior of different noise components with respect to the noise scale $\sigma$.
    }
    \label{fig:noise_norm}
    \vspace{-5mm}
\end{wrapfigure}

\textbf{DP noise amplification in LoRA \cite{sun2024improving}.} When DP-SGD is applied to LoRA parameters, Gaussian noise is added
to the gradients.
Let $\Delta B$ and $\Delta A$ denote the updates to $B$ and $A$ in a given round,
which consist of a true update component and a DP noise component:
$
\Delta B = \Delta B^{\star} + \mathcal{N}_B, \quad
\Delta A = \Delta A^{\star} + \mathcal{N}_A.
$
The resulting update to the effective weight matrix is given by
$\Delta W=B \Delta A^{\star}+\Delta B^{\star} A+B \mathcal{N}_A+\mathcal{N}_B A+\mathcal{N}_B \mathcal{N}_A$.
As illustrated in Figure~\ref{fig:noise_norm}, different noise components exhibit distinct scaling behaviors with respect to the noise scale $\sigma$. 
The full-model noise $\|\mathcal{N}_W\|_F$ grows linearly, whereas the multiplicative cross-term $\|\mathcal{N}_B \mathcal{N}_A\|_F$ exhibits quadratic growth and becomes dominant as $\sigma$ increases. 
This phenomenon leads to \emph{noise amplification} in the effective model update, which can significantly degrade convergence and stability.

Additional related works are introduced in Appendix~\ref{sec:Appendix A}.

\vspace{-2mm}

\section{Method}

\vspace{-2mm}

In this section, we introduce our AS-LoRA framework. We first describe the problem formulation of FL with LoRA, and then present the score computation algorithm $\mathcal{S}(\cdot)$ and the mode selection algorithm $\mathcal{M}(\cdot)$ that together determine which LoRA component is activated at each layer and communication round. The full algorithmic procedure and visual illustration of AS-LoRA are deferred to Appendix~\ref{sec:Appendix B}.

\textbf{Problem formulation.}
We consider a FL setting with $K$ clients and a central server. Each client holds a local dataset that is not shared with others, and the goal is to collaboratively fine-tune a large pre-trained model while preserving data privacy. Following Section \ref{sec:back}, the adapted weight matrix is parameterized as \(W = W_0 + BA\), where \(A\) and \(B\) denote the two low-rank components.
Since LoRA is applied to multiple layers of the model, we further decompose the update into layer-wise adapters. Specifically, for the $n$-th layer at round $t$, the adapted weight can be written as
$
W_t(n) = W_0(n) + B^t(n)A^t(n),~ n \in \left[N \right],
$
where $N$ denotes the total number of LoRA layers.

\textbf{Score computation algorithm.} To determine which LoRA component should be activated at each communication round, we design a score function based on a second-order approximation of the expected loss decrease after a gradient update. Let $a \in \{A, B\}$ denote a candidate LoRA component. For a given layer $n$, we consider one gradient descent step at round $t$,
$
W' = W_t - \eta \nabla_{a(n)} L(W_t),
$
whose resulting change in the loss admits the second-order Taylor approximation
$
L(W') \approx L(W_t) 
- \eta \|\nabla_{a(n)} L(W_t)\|^2 
+ \frac{\eta^2}{2} \nabla_{a(n)} L(W_t)^\top H_{a(n)}(W_t) \nabla_{a(n)} L(W_t),
$
where $H_{a(n)}(W_t)$ is the Hessian with respect to the parameter block $a(n)$. This motivates the following score function:
\begin{equation}
\label{eq:score_function} 
\mathcal{S}_{a(n)}(W_t) =
\|\nabla_{a(n)} L(W_t)\|^2 
- \frac{\eta}{2}\, \nabla_{a(n)} L(W_t)^\top H_{a(n)}(W_t)\, \nabla_{a(n)} L(W_t),
\end{equation}
where we omit the client index $k$ throughout this derivation for notational clarity. The two terms in $\mathcal{S}_{a(n)}$ admit complementary interpretations from two well-established design principles. The \emph{first term} extends the classical Gauss--Southwell rule of block coordinate descent (BCD)~\cite{nesterov2012bcd}, which selects the block with the largest gradient magnitude and yields strictly faster convergence than uniform block selection. The \emph{second term} introduces a curvature-induced penalty whose effect mirrors the implicit bias of sharpness-aware minimization (SAM)~\cite{foret2021sam}, which discourages updates along directions in which a gradient step would overshoot a sharp region of the loss surface. Selecting the block that maximizes $\mathcal{S}_{a(n)}$ therefore corresponds to choosing the LoRA component that offers the largest expected one-step loss decrease while preferring flatter regions of the loss landscape. We formally establish these two design benefits in Theorem~\ref{thm:conv} and Theorem~\ref{thm:implicit}.

Now, we introduce the Gaussian random projection \cite{johnson1984extensions}, which is applied to the computed score \eqref{eq:score_function}. Recent studies have shown that large language models often exhibit channel-wise activation outliers  \cite{xiao2023smoothquant, lin2024awq}. \cite{jung2025gralora} further shows that the gradient magnitude of the $A$ component is highly concentrated in a small subset of channels, and the squared-gradient term $\|\nabla_{a(n)} L(W_t)\|^2$ in \eqref{eq:score_function}) can lead to an overestimation of the score of the $A$ component. To alleviate this, we apply the random projection step as follow.
For each LoRA component, we construct a random projection matrix in the original gradient space that preserves the original dimensionality while redistributing channel-wise outliers across coordinates, denoted as
$
R_A(n), 
R_B(n).
$
For client $k$ at round $t$, let
$
g_{a,k}^t(n) = \nabla_{a(n)} \mathcal{L}_k^t, 
$
denote the locally computed gradients with respect to the LoRA components of layer $n$. 
We project these gradients as
$
\widetilde{g}_{A,k}^t(n) = g_{A,k}^t(n) R_A(n), \quad
\widetilde{g}_{B,k}^t(n) = R_B(n) g_{B,k}^t(n),
$
so that the projected gradients redistribute channel-wise outliers across dimensions. The effectiveness of random projection is empirically validated in Appendix~\ref{app:random_projection}. In additoin, we remark that our score computation is \emph{single-pass}, \emph{i.e.,} the scores are obtained as a by-product of local training phase. As an alternative, one could consider a \emph{two-pass inner-loop computation} that re-evaluates the score after the globally aggregated model is redistributed to the clients. We empirically compare the two designs in terms of performance, and overhead (Appendix~\ref{app:pass}).

\textbf{Mode selection algorithm.} We now design a mode selection algorithm, denoted by $\mathcal{M}(\cdot)$ which determines the active LoRA component for each layer at round $t$ using the layer-wise scores calculated from stored scores and aggregate scores in an Exponential Moving Average (EMA) manner. The algorithm outputs the \emph{mode vector} indicating which component should be optimized.  
We denote the mode vector at round $t$ as
$
\{M^t(1), M^t(2), \dots, M^t(N)\},~ M^t(n) \in \{0,1\},
$
where $M^t(n)=0$ indicates that $A^t(n)$ is optimized and $M^t(n)=1$ indicates that $B^t(n)$ is optimized during round $t$. 
Accordingly, we define the layer-wise scores as $
\hat{S}_a^t(n) := \mathrm{EMA}\!\left(\hat{S}_a^{t-1}(n), S_a^t(n)\right),
$
where $S_a^t(n):=\tfrac{1}{K}\sum_{k=1}^{K} S_{a,k}^t(n)$. Instead of deterministically selecting the component with the larger score, we use a temperature-scaled softmax strategy. 
This approach encourages exploration during training and is consistent with the common practice of temperature annealing \cite{jang2016categorical, maddison2016concrete}. The probability of selecting $a$ at round $t$ is defined as
$
P^t(a,n)=
\frac{\exp\!\left(\hat{S}_a^{t}(n)/T^t\right)}
{\exp\!\left(\hat{S}_A^{t}(n)/T^t\right)+\exp\!\left(\hat{S}_B^{t}(n)/T^t\right)},
$
where $T^t$ is the temperature parameter at round $t$.
The temperature follows an annealing schedule
\begin{equation}
T^t =
\begin{cases}
T^0, & t \le t_w, \\
\max\!\left(T_{\min},\, T^0 \cdot \gamma^{(t-t_w)}\right), & t > t_w ,
\end{cases}
\end{equation}
where $T^0$ is the initial temperature, $T_{\min}$ is the minimum temperature, $\gamma \in (0,1)$ is the decay factor, and $t_w$ denotes the warm-up period.
During the warm-up stage ($t \le t_w$), the two LoRA components are alternately activated across rounds. The stochastic selection mechanism prevents premature convergence to a single component while allowing the framework to gradually favor the component with consistently higher scores as the temperature decreases. We remark that our mode selection operates on scores aggregated by \emph{uniform averaging} across clients, which gives equal weight to every client and and can be regarded as the standard choice. In Appendix~\ref{app:selection_rule}, we compare this design against two alternative aggregation rules, namely \emph{majority voting} and \emph{weighted averaging}.

\vspace{-3mm}

\section{Theoretical Analysis}\label{sec:theory}
We analyze AS-LoRA along two axes: $\emph{(i) layer-wise freedom}$ which is the ability to activate $A$ or $B$ independently per layer and $\emph{(ii) score-driven selection}$ which is  $\arg\max_a S_a$ with $S_a=\|g_a\|^2-\tfrac{\eta}{2}g_a^\top H_a g_a$. Each axis contributes a distinct and quantifiable advantage over the layer-tied or static baselines including FFA-LoRA \cite{sun2024improving}, and RoLoRA \cite{chen2025robust}. To focus our goals, we do not consider the effect of random projection, EMA, temperature-scaled softmax strategy, and the variance of training and scoring caused by distributed settings in our analysis. We first prove that AS-LoRA requires no extra privacy cost. In fact, AS-LoRA's selection at each layer depends on $\|g_a^{(n)}\|^2$ and $(g_a^{(n)})^\top H_a^{(n)} g_a^{(n)}$, both of which are gradient-dependent statistics that in principle \emph{could} be privacy-relevant. The following theorem resolves the concern based on the post-processing invariance of DP \cite{dwork2014dpfoundations}. The proof of Theorem \ref{thm:privacy} can be found in Appendix \ref{app:proof_of_privacy}.

\begin{theorem}[$(\epsilon,\delta)$-DP at no privacy cost]\label{thm:privacy}
With per-round DP-SGD~\cite{abadi2016deep} (subsampled Gaussian, sensitivity $C$, noise scale $\sigma$, sampling rate $q$), AS-LoRA satisfies $(\epsilon,\delta)$-DP with
\(
\epsilon \le \tfrac{q^2 C^2 T}{\sigma^2}\log(1/\delta),
\)
the same bound with FFA-LoRA/RoLoRA via moments accountant / R\'enyi-DP composition~\citep{mironov2017renyi,wang2019sgm}.
\end{theorem}
Now, we show that any schedule which ties the mode across all layers must incur an irreducible loss floor, whereas AS-LoRA provably eliminates it. For the analysis, we first introduce \emph{layer-wise population-optimal LoRA pair} that minimizes the \emph{population} fine-tuning loss $f_n$ as
\begin{equation}\label{def_optimal}
\bigl(A^\star(n),\,B^\star(n)\bigr) :=\arg\min_{A\in\R^{r\times d},\,B\in\R^{d\times r}}\; \E_{X_{n}}\!\bigl[\ell\bigl(f(X_{n})\bigr)\bigr].
\end{equation}
Next, we define the \emph{subspace misalignment} $\delta^t(n)$ as
\begin{equation}\label{def_misalign}
    \delta^t(n) :=\bigl\|B^\star(n)\,A^\star(n)\,(I - P_{A^t(n)})\bigr\|_F  = \bigl\|\Delta W^\star(n)(I - P_{A^t(n)})\bigr\|_F
\in[0,\|\Delta W^\star(n)\|_F],
\end{equation}
where $P_{A^t(n)}:=A^t(n)^\dagger A^t(n)\in\R^{d\times d}$ denote the orthogonal projector onto $\mathrm{row}(A^t(n))$, $A^\dagger$ denotes the Moore-Penrose pseudoinverse \cite{barata2012moore}. It quantifies the part of the optimal update $B^\star A^\star$ that lies outside $\mathrm{row}(A^t(n))$. The intermediate values with $r=1$ was used by \cite{chen2025robust} to diagnose the fundamental obstruction of FFA-LoRA, where $A^t(n)\!\equiv\!A^0(n)$ so $\delta^t(n)$ persists forever. With $X_n\in\R^{d\times m}$ stacking the layer-$n$ inputs, we define \emph{per-layer reconstruction risk} as
\begin{equation} \label{def_recon}
    \mathcal R_\text{rec}(\mathbf W) :=\frac{1}{N}\sum_{n=1}^{N}\bigl\|B^\star(n)A^\star(n)X_n - B(n)A(n)X_n\bigr\|_F^2,
\end{equation}
\emph{i.e.} the average layer-wise squared Frobenius distance between each layer's effective update and its optimal LoRA pair defined in \eqref{def_optimal}. While the reconstruction risk with one layer was also used in \cite{chen2025robust}, $\mathcal R_\text{rec}$ in \eqref{def_recon} is averaged layer-wise. Following Theorem \ref{thm:irreducible} and Corollary \ref{cor:asstrict} prove the strict dominance of layer-wise freedom over layer-tied schedules.

\begin{theorem}[Irreducible reconstruction floor of layer-tied schedules]\label{thm:irreducible}
Under Assumptions~\ref{asm:smooth}--\ref{asm:hess} and a well-conditioned input regime that is $\sigma_{\max}(X_n)/\sigma_{\min}^+(X_n A^0(n)^\dagger) = O(1)$ uniformly over layers $n$, FFA-LoRA~\citep{sun2024improving} which freezes $A^{(n)}$ at the random initialization $A^0(n)$ satisfies
\begin{equation}
\lim_{T\to\infty}\E\!\left[\mathcal R_\text{rec}(\mathbf W_T^{\text{FFA}})\right] \;=\; \Theta\!\left(\tfrac{1}{N}\!\sum_{n=1}^N\bigl(\delta^0(n)\bigr)^2\right),    
\end{equation}
where $\delta^0(n)\!=\!\|B^\star(n)A^\star(n)(I\!-\!P_{A^0(n)})\|_F$ is the initial subspace misalignment, and the implicit $\Theta(1)$ constants depend on the per-layer conditioning $\sigma_{\max}^2(X_n)/\sigma_{\min}^2(X_n A^0(n)^\dagger) = O(1)$ and absorb both the correction $O(1/\sqrt{Km})$ and the SGD noise floor $O(\eta\sigma_g^2)$. Moreover, any RoLoRA-type layer-tied alternation inherits $\Omega\!\bigl(\tfrac{1}{N}\sum_n(\delta^0(n))^2\bigr)$ floor under adversarial layer configurations.
\end{theorem}

The preliminaries can be found in Appendix \ref{app:prelim}, the proof of Theorem~\ref{thm:irreducible} can be found in Appendix~\ref{app:proof_of_irreducible}, and the proof of Corollary \ref{cor:asstrict} can be found in Appendix \ref{app:proof-asstrict}.

\begin{corollary}[Strict improvement of AS-LoRA]\label{cor:asstrict}
Under the same well-conditioned input regime as Theorem~\ref{thm:irreducible}: if AS-LoRA activates $A^{(n)}$ on $\Theta(T)$ rounds per layer, then for any $\varepsilon>0$ there exists $T=O(\log(1/\varepsilon))$ with $\delta^T(n)\!\le\!\varepsilon$ for all $n$, and
\begin{equation}
    \lim_{T\to\infty}\E\bigl[\mathcal R_\text{rec}(\mathbf W_T^{\text{AS}})\bigr] \;=\; O\!\bigl(\varepsilon^2\bigr) \;<\; \lim_{T\to\infty}\E\bigl[\mathcal R_\text{rec}(\mathbf W_T^{\text{FFA}})\bigr] \;=\; \Theta\!\bigl(\tfrac{1}{N}\!\sum_n(\delta^0(n))^2\bigr)
\end{equation}
whenever any $\delta^0(n_0)>\varepsilon$ for at least one layer $n_0$ (the FFA floor's $\Theta(1)$ constant is bounded below).
\end{corollary}

Having established that per-layer freedom is the right design axis, we now ask which block to pick at each round. Here, we discuss the improvement via the selection based on $S_a^{(n)}$. For this, we first define the \emph{adaptive gain} $\Delta_{\mathcal S}^t(n) := \E\!\left[\max_{a\in\{A,B\}} S_a^t (n)\right] -\frac{1}{2}\,\E\!\left[S_A^{t}(n)+S_B^{t}(n)\right] $
the per-layer gap between the score of AS-LoRA's greedy pick and the uniform baseline. This is the score-space version of the \emph{suboptimality gap} used in Gauss--Southwell analyses of BCD \cite{nesterov2012bcd}. In addition, We collect $1-\tfrac{\eta}{2}\lambda_a^{(n)}$ across the trajectory and across layers, and define $\rho^\star := \inf_{t,n}\Bigl(1-\tfrac{\eta}{2}\lambda_{M^t(n)}^{(n)}\Bigr) \;\in\;(0,1]$,
the \emph{curvature-penalty ratio}, where $\lambda_{a}^{(n)}:=(g_{a}^{(n)})^{\top} H_a^{n} g_a^{(n)}/\|g_a^{(n)}\|^2$ the block Rayleigh quotient. Intuitively, $\rho^\star$ is the fraction of pure-descent score that survives after the curvature penalty under the AS-LoRA selection rule. Note that the curvature-penalty ratio with AS-LoRA is larger than the value with static schedules in the late-training regime where $\|g_A^{(n)}\|^2 \approx \|g_B^{(n)}\|^2$. The following theorem provides the convergence of AS-LoRA and improvement compared to the static schedules characterized by the factors $\overline{\Delta_{\mathcal S}}$, and $\rho^\star$   
\begin{theorem}[Accelerated Convergence with quantities $\overline{\Delta_{\mathcal S}}$, and $\rho^\star$]\label{thm:conv}
Under Assumptions~\ref{asm:smooth}--\ref{asm:hess} with $\eta\le 1/\beta$,
\begin{equation} \label{non-convex_bound}
\min_{t\le T}\E\|\nabla\Loss(W_t)\|^2 \;\le\; \frac{2(\Loss(W_0)-\Loss^\star)}{\eta\,\rho^\star\,T} \;-\; \frac{2\,\overline{\Delta_{\mathcal S}}}{\rho^\star} \;+\; \frac{\eta\beta\sigma_g^2}{\rho^\star},    
\end{equation}
where $\overline{\Delta_{\mathcal S}}:=\tfrac{1}{T}\sum_{t=1}^{T}\tfrac{1}{N}\sum_{n=1}^{N}\Delta_{\mathcal S}^t(n)$. Under the additional PL condition (Assumption~\ref{asm:pl})~\citep{karimi2016pl},
$
\E[\Loss(W_T)-\Loss^\star]\le(1-\eta\mu\rho^\star)^T(\Loss(W_0)-\Loss^\star)+O(\eta\sigma_g^2/\mu).
$
\end{theorem}

The non-convex bound \eqref{non-convex_bound} recovers the standard $O(1/(\eta\rho^\star T))$ gradient-norm rate of SGD \citep{karimi2016pl} but with two distinguishing features. The terms are multiplied by $1/\rho^{\star}$, and the second term $-2\overline{\Delta_{\mathcal S}}/\rho^\star$ guarantees that AS-LoRA's pick tightens the bound directly and beats the uniform baseline. Under the PL condition, the improvement compounds geometrically, yielding a $(1-\eta\mu\rho^\star)^T$ rate that matches the best-known LoRA fine-tuning rate~\citep{liu2022loss,hayou2024lora+} but with $\mu\rho^\star$ in place of $\mu$, i.e.\ AS-LoRA provably converges generally faster than any layer-tied or static baselines. The proof can be found in Appendix \ref{app:proof-conv}. Finally, we discuss that AS-LoRA's trajectory with $S_a^{(n)}$ is biased toward the \emph{flatter} of the two blocks, which can be seen as an implict SAM-style regularization \cite{foret2021sam} or cheap method to achieve the flatness without the SVD overhead in FedSVD \cite{lee2025fedsvd}. The proof can be found in Appendix \ref{app:proof-implicit}.

\begin{theorem}[Implicit flatness via curvature-penalty score]\label{thm:implicit}
For each layer $n$, let $\overline H_t^{(n)} := \E[H_{M^t(n)}^{(n)}]$ denote the expected block Hessian at the selected mode under AS-LoRA's argmax rule. Under Assumptions~\ref{asm:smooth}--\ref{asm:hess},
\begin{equation}    
\E\bigl[\lambda_{\max}\bigl(\overline H_t^{(n),\text{AS}}\bigr)\bigr] \;\le\; \E\bigl[\lambda_{\max}\bigl(\overline H_t^{(n),\text{static}}\bigr)\bigr] \quad\text{for every } n\in[N],
\end{equation}
with strict inequality at every layer at which the two per-block top-spectra $\{\lambda_{\max}(H_A^{(n)}), \lambda_{\max}(H_B^{(n)})\}$ differ. AS-LoRA's trajectory is therefore biased toward the \emph{flatter} of the two blocks, yielding an implicit SAM-style regularization \citep{foret2021sam}.
\end{theorem}

\vspace{-3mm}

\section{Experiments}

\textbf{Experimental setup.} \label{sec:experimental_setup}
For language tasks, we use RoBERTa-large \cite{liu2019roberta} as the backbone model, while ViT-large \cite{dosovitskiy2020image} is used for image tasks. 
The LoRA rank is set to $r=8$, and the scaling factor is set to $\alpha=8$. 
We consider a federated setting with $K=6$ clients for language tasks and $K=3$ clients for the vision tasks, with full client participation in every communication round.
Data are partitioned using a Dirichlet distribution with $\alpha=0.5$, except for SQuAD \cite{rajpurkar2016squad, rajpurkar2018know}, where data are equally partitioned across clients.
Training is conducted for $T=100$ communication rounds, with $\tau=10$ local steps per round. 
Following the empirical setup of \cite{sun2024improving}, the learning rate $\eta$ is selected from $\{0.05, 0.1, 0.25\}$ for FedLoRA and from $\{0.1, 0.25, 0.5, 1.0\}$ for all other methods. 
The batch size is set to $B=128$. 
For DP-SGD, we set the target failure probability to $\delta = 10^{-5}$ and fix the clipping norm at $C=2$. 
The differential privacy budget $\epsilon$ is computed using Opacus 
\cite{yousefpour2021opacus}. All experiments were run on RTX Pro $6000$ GPUs. Additional details are provided in Appendix~\ref{app:experimental_detail}. Unless otherwise specified, all experiments follow the settings described 
in Section~\ref{sec:experimental_setup} and Appendix~\ref{app:experimental_detail}.


\begin{wraptable}{r}{0.38\linewidth}
\vspace{-12pt}
\centering
\scriptsize
\setlength{\tabcolsep}{5.5pt}
\renewcommand{\arraystretch}{1.15}
\newcommand{\std}[1]{\,{\tiny$\pm$}\,\tiny #1}
\caption{Comparison on MNLI under DP-SGD ($\epsilon = 3$).}
\label{tab:mnli_variant}
\begin{tabular}{lcc}
\toprule
\textbf{Method} & \textbf{MNLI (m)} & \textbf{MNLI (mm)} \\
\midrule
Proposed  (FD)       & 79.15\std{0.23} & 80.03\std{0.40} \\
Proposed (HVP)   & 80.29\std{0.84} & 81.29\std{0.80} \\
\bottomrule
\vspace{-20pt}
\end{tabular}
\end{wraptable} 
For the proposed AS-LoRA, the warm-up ratio is set to $0.1$.  The temperature parameters are set to $T^0 = 2.0$, $T_{\min} = 0.2$, and $\gamma = 0.95$. While the proposed scoring function provides a theoretically grounded criterion,
direct Hessian computation is computationally prohibitive at the scale of large
pretrained models. To make the curvature-aware score practical, we consider two
approximation strategies, Hessian-Vector Products (HVP) \cite{pearlmutter1994fast} and Finite-Difference (FD) \cite{nocedal2006numerical} estimation, which are described in Appendix \ref{sec:Appendix B}. As shown in Table~\ref{tab:mnli_variant}, HVP-based scoring
achieves better performance, whereas FD-based scoring yields competitive
results without explicitly forming or storing the Hessian matrix. We
adopt the FD-based criterion in all subsequent experiments, as it preserves the
curvature-aware nature of the proposed score while maintaining practical
computational cost.

\textbf{Baselines.} We compare our method with the following baselines, which directly align with the design objective of AS-LoRA. Methods that address only the aggregation error~\cite{wang2024flora, singhal2024fedex, bai2024federated} are excluded from comparison, as they have been shown to suffer significant performance degradation under DP noise~\cite{lee2025fedsvd}.\\
\phantom{a}\textbf{FedLoRA} \cite{zhang2024towards}: Both $A$ and $B$ are locally trained and independently aggregated at the server.\\
\phantom{a}\textbf{FFA-LoRA} \cite{sun2024improving}: $A$ is frozen, and only $B$ is locally trained and aggregated.\\
\phantom{a}\textbf{RoLoRA} \cite{chen2025robust}: In odd-numbered rounds, $A$ is frozen and only $B$ is locally trained and aggregated.\\\phantom{a}In even-numbered rounds, $B$ is frozen and only $A$ is locally trained and aggregated.

\begin{table*}[h]
\centering
\scriptsize
\setlength{\tabcolsep}{3pt}
\renewcommand{\arraystretch}{1.05}
\newcommand{\std}[1]{\,{\tiny$\pm$}\,\tiny #1}

\caption{Performance comparison on the GLUE benchmark \cite{wang2018glue} under different privacy budgets.}
\label{tab:glue_dp}

\resizebox{\textwidth}{!}{
\begin{tabular}{c l c c c c c c c}
\toprule
\textbf{DP Budget} & \textbf{Method} & \textbf{QNLI} & \textbf{MNLI (m)} & \textbf{MNLI (mm)} & \textbf{SST-2} & \textbf{QQP}  & \textbf{SNLI} & \textbf{AVG}\\
\midrule

\multirow{4}{*}{$\epsilon = 3$}
 & FedLoRA   & \underline{72.43}\std{10.56} & 63.76\std{8.24} & 64.98\std{8.65} & 91.37\std{1.63} & 67.86\std{10.47} & 78.19\std{2.04} & 73.10\std{10.36} \\
 & FFA-LoRA  & 68.18\std{11.72} & 65.06\std{5.40} & 67.51\std{5.34} & 91.00\std{0.98} & \underline{75.79}\std{1.28} & 77.91\std{7.49} & 74.24\std{8.70} \\
 & RoLoRA    & 65.02\std{13.32} & \underline{69.15}\std{4.50} & \underline{71.02}\std{3.93} & \underline{91.37}\std{0.63} & 72.86\std{1.21} & \underline{82.22}\std{1.49} & \underline{75.27}\std{8.36} \\
 & Proposed  & \textbf{81.06}\std{1.34} & \textbf{79.15}\std{0.23} & \textbf{80.03}\std{0.40} & \textbf{92.86}\std{0.44} & \textbf{78.51}\std{0.38} & \textbf{85.23}\std{0.77} & \textbf{82.81}\std{4.93} \\
\midrule

\multirow{4}{*}{$\epsilon = 1$}
 & FedLoRA   & 55.72\std{1.40} & 58.03\std{8.05} & 59.37\std{8.94} & 57.57\std{0.37} & 65.72\std{9.90} & 74.79\std{1.63} & 61.87\std{6.56} \\
 & FFA-LoRA  & \textbf{71.95}\std{7.92} & \underline{62.70}\std{7.67} & \underline{65.04}\std{7.72} & \underline{79.27}\std{11.71} & 71.00\std{2.41} & \underline{81.37}\std{1.57} & \underline{71.89}\std{6.56} \\
 & RoLoRA    & 53.85\std{5.34} & 59.91\std{10.66} & 61.37\std{10.73} & 62.99\std{13.43} & \underline{73.10}\std{1.59} & 80.97\std{0.79} & 65.37\std{9.26} \\
 & Proposed  & \underline{71.52}\std{12.50} & \textbf{76.71}\std{0.92} & \textbf{77.54}\std{0.77} & \textbf{80.76}\std{14.08} & \textbf{74.94}\std{1.78} & \textbf{84.32}\std{0.64} & \textbf{77.63}\std{4.10} \\
\bottomrule
\end{tabular}
}
\vspace{-16pt}
\end{table*}

\begin{table*}[h]
\centering
\scriptsize
\setlength{\tabcolsep}{2pt}
\renewcommand{\arraystretch}{1.05}
\newcommand{\std}[1]{\,{\tiny$\pm$}\,\tiny #1}

\caption{Performance comparison on SQuAD benchmarks \cite{rajpurkar2016squad, rajpurkar2018know},CIFAR-100 \cite{krizhevsky2009learning} and Tiny-imageNet \cite{deng2009imagenet} under different privacy budgets.}
\label{tab:squad_vision_dp}
\begin{tabular}{c l cc cc cc}
\toprule
\textbf{DP Budget} & \textbf{Method} 
& \multicolumn{4}{c}{\textbf{Question Answering}} 
& \multicolumn{2}{c}{\textbf{Image Classification}} \\

\cmidrule(lr){3-6} \cmidrule(lr){7-8}
& 
& \textbf{SQuAD v1.1 (EM)} & \textbf{SQuAD v1.1 (F1)} 
& \textbf{SQuAD v2.0 (EM)} & \textbf{SQuAD v2.0 (F1)} 
& \textbf{CIFAR-100} & \textbf{Tiny-ImageNet} \\

\midrule

\multirow{3}{*}{$\epsilon = 3$}
& FFA-LoRA  
& \underline{72.30}\std{3.52} & \underline{82.83}\std{2.83} 
& \underline{55.90}\std{3.66} & \underline{59.92}\std{3.93} 
& 87.37\std{0.26} & \underline{85.50}\std{0.15} \\

& RoLoRA    
& 69.75\std{2.70} & 80.63\std{2.30} 
& 53.58\std{1.22} & 56.88\std{1.16} 
& \underline{87.75}\std{0.40} & 84.62\std{0.13} \\

& Proposed  
& $\textbf{73.98}$\std{1.36} & $\textbf{83.99}$\std{1.17} 
& $\textbf{60.59}$\std{2.97} & $\textbf{64.23}$\std{3.62} 
& $\textbf{89.55}$\std{0.20} & $\textbf{86.66}$\std{0.32} \\

\midrule

\multirow{3}{*}{$\epsilon = 1$}
& FFA-LoRA  
& \underline{70.42}\std{3.91} & \underline{81.24}\std{3.26} 
& \underline{54.33}\std{3.91} & \underline{58.58}\std{4.16} 
& \underline{84.79}\std{0.23} & \underline{83.52}\std{0.26} \\

& RoLoRA    
& 61.69\std{3.81} & 73.94\std{3.10} 
& 51.48\std{1.47} & 53.59\std{1.80} 
& 84.64\std{0.38} & 82.34\std{0.14} \\

& Proposed  
& $\textbf{72.78}$\std{1.50} & $\textbf{83.23}$\std{1.37} 
& $\textbf{58.09}$\std{2.57} & $\textbf{61.64}$\std{3.16} 
& $\textbf{85.62}$\std{0.24} & $\textbf{83.71}$\std{0.13} \\

\bottomrule
\end{tabular}
\vspace{-0pt}
\end{table*}

\textbf{Comparison results.} All results are reported as mean $\pm$ standard deviation over $4$ runs with different seeds. Best results are shown in \textbf{bold} and second-best results are \underline{underlined}. Table~\ref{tab:glue_dp} reports the performance on the GLUE benchmark \cite{wang2018glue} under different privacy budgets. The proposed method consistently outperforms all baselines across tasks.
Under a privacy budget of $\epsilon = 3$, AS-LoRA achieves the best average score of $\mathbf{82.81}$, outperforming RoLoRA ($75.27$) by +$\mathbf{7.54}$ percentage points (pp). 
Under a budget of $\epsilon = 1$, where training instability is more severe, AS-LoRA attains $\mathbf{77.64}$, surpassing FFA-LoRA ($\mathbf{71.89}$) by +$\mathbf{5.74}$ pp. In particular, significant improvements are observed on MNLI (mm) (+$12.50$ pp over FFA-LoRA) and consistent gains across SNLI and QQP.
As shown in Table~\ref{tab:squad_vision_dp}, these improvements extend beyond classification tasks. On SQuAD, AS-LoRA achieves higher performance across both v1.1 and v2.0 metrics, e.g., +$1.35$ EM and +$1.96$ F1 over FFA-LoRA on SQuAD v1.1 under $\epsilon = 3$. Moreover, the method consistently improves vision benchmarks, achieving +$2.34$ pp on CIFAR-100 and +$1.19$ pp on Tiny-ImageNet. Extended results on LoRA rank and heterogeneity of data distribution can be found in Appendix~\ref{app:alpha_rank}.

\begin{wraptable}{r}{0.52\linewidth}
\vspace{-13pt}
\centering
\scriptsize
\setlength{\tabcolsep}{2pt}
\renewcommand{\arraystretch}{1.05}
\newcommand{\std}[1]{\,{\tiny$\pm$}\,\tiny #1}

\caption{Comparison across mode selection granularity.(\textit{Per-client}: layer-wise and client-wise, \textit{AS-LoRA}: layer-wise and client-tied, \textit{Global}: layer-tied and client-tied)}
\label{tab:mode_granularity}

\begin{tabular}{c l c c c}
\toprule
\textbf{DP Budget} & \textbf{Method} & \textbf{MNLI (m)} & \textbf{MNLI (mm)} & \textbf{QQP} \\
\midrule

\multirow{3}{*}{$\epsilon$ = 3}
& Per-client & 72.73\std{2.07} & 73.73\std{2.23} & 63.18\std{0.00} \\
& AS-LoRA (ours) & $\textbf{79.15}$\std{0.23} & $\textbf{80.03}$\std{0.40} & $\textbf{78.51}$\std{0.38} \\
& Global & 75.74\std{2.53} & 76.98\std{2.44} & 73.61\std{2.75} \\

\midrule

\multirow{3}{*}{$\epsilon$ = 1}
& Per-client & 34.60\std{1.72} & 34.53\std{1.37} & 63.18\std{0.00} \\
& AS-LoRA (ours) & $\textbf{76.71}$\std{0.92} & $\textbf{77.54}$\std{0.77} & $\textbf{74.94}$\std{1.78} \\
& Global & 72.54\std{3.85} & 73.79\std{3.87} & 72.62\std{3.26} \\
\bottomrule
\end{tabular}
\vspace{-15pt}
\end{wraptable}

\textbf{Granularity of mode selection scope.}
While AS-LoRA performs layer-wise mode selection that is shared across all clients, we evaluate two alternatives to justify this design choice. Both are round-wise adaptive as AS-LoRA, but differ in selection granularity. \textit{Global} shares a single mode across all layers and clients, whereas \textit{Per-client} selects modes independently for each layer \emph{and} each client. As shown in Table~\ref{tab:mode_granularity}, \textit{Per-client}, yields the worst performance across all tasks and privacy budgets. It is because the aggregated model suffers from aggregation error. \textit{Global} also underperforms AS-LoRA. Sharing a single mode across all layers prevents the model from exploiting layer-specific dynamics. AS-LoRA strikes the right balance, reserving model expressiveness without compromising aggregation stability.









\begin{wraptable}{r}{0.49\linewidth}
\vspace{-13pt}
\centering
\scriptsize
\setlength{\tabcolsep}{2pt}
\renewcommand{\arraystretch}{1.05}
\newcommand{\std}[1]{\,{\tiny$\pm$}\,\tiny #1}

\caption{Comparison with FedSVD \cite{lee2025fedsvd} in performance and aggregation time.}
\label{tab:analysis_svd}

\begin{tabular}{c l c c c}
\toprule
\textbf{DP Budget} & \textbf{Method} & \textbf{Time (s)} & \textbf{MNLI (m)} & \textbf{MNLI (mm)} \\
\midrule

\multirow{4}{*}{$\epsilon = 3$}
& FedSVD(10) & 1.00 (33$\times$) & 78.16\std{1.01} & 79.43\std{1.21} \\
& FedSVD(5)  & 1.80 (60$\times$) & 77.06\std{1.47} & 79.07\std{1.10} \\
& FedSVD(1)  & 5.44 (181$\times$) & 78.49\std{0.89} & 79.75\std{0.68} \\
& Proposed   & \textbf{0.03 (1$\times$)} & \textbf{79.15}\std{0.23} & \textbf{80.03}\std{0.40} \\
\midrule

\multirow{4}{*}{$\epsilon = 1$}
& FedSVD(10) & 1.00 (33$\times$) & 74.24\std{1.48} & 75.69\std{1.45} \\
& FedSVD(5)  & 1.80 (60$\times$) & \textbf{76.97}\std{1.25} & \textbf{78.16}\std{1.13} \\
& FedSVD(1)  & 5.44 (181$\times$) & 76.62\std{0.77} & 77.79\std{0.60} \\
& Proposed   & \textbf{0.03 (1$\times$)} & 76.71\std{0.92} & 77.54\std{0.77} \\
\bottomrule
\end{tabular}

\vspace{-10pt}
\end{wraptable}

\textbf{Comparison with SVD-based method.} 
Table~\ref{tab:analysis_svd} compares AS-LoRA with \emph{FedSVD} \cite{lee2025fedsvd} the representative SVD-based method. FedSVD improves performance by introducing an additional SVD computation at the server. FedSVD$(n)$ applies SVD at every $n$-th communication rounds with standard FedAvg used otherwise.
Under $\epsilon=3$, the proposed method achieves the best performance while requiring substantially lower aggregation time. Under $\epsilon=1$, FedSVD$(5)$ shows slightly higher performance, but the gap is marginal. 
Moreover, while the proposed method relies on the FD-based score, it can scale to heavier approximations (HVP) as shown in Table~\ref{tab:mnli_variant}.

\begin{figure*}[h]
    \centering
    \includegraphics[width=1.0\textwidth]{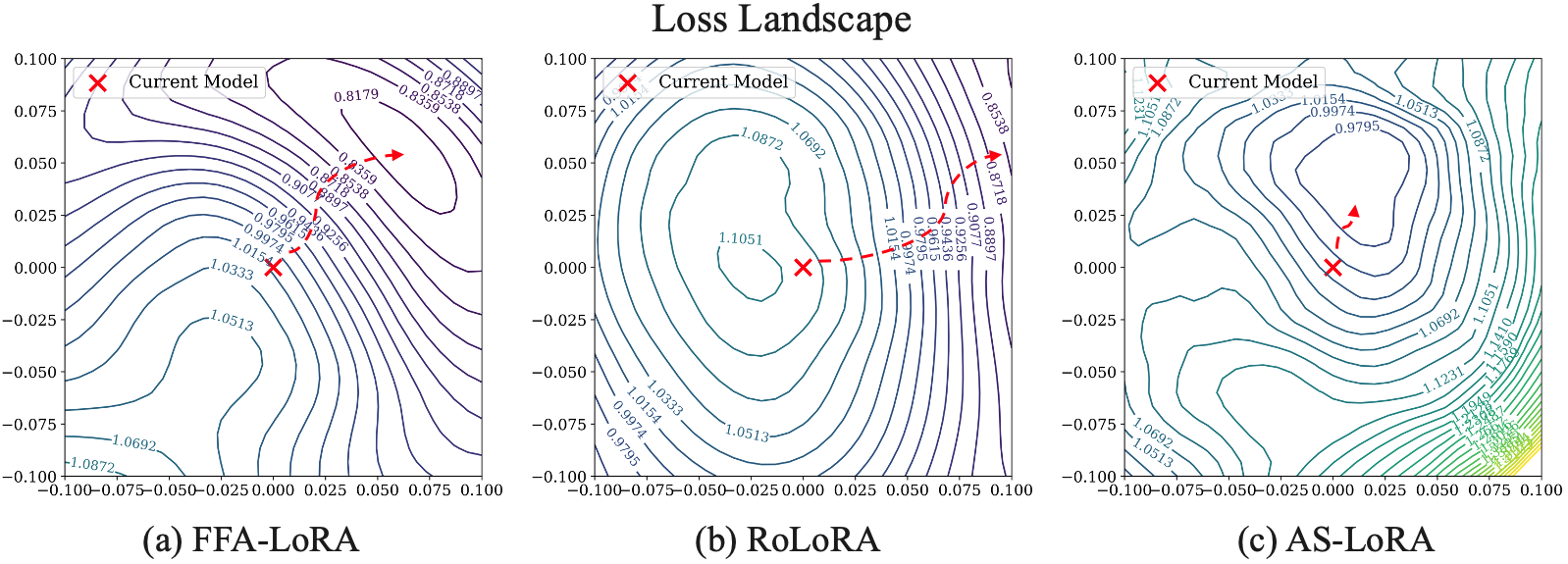}
    \caption{Loss landscape visualization of FFA-LoRA, RoLoRA, and AS-LoRA trained on MNLI with $T=100$ and $\tau=10$ under DP-SGD ($\epsilon=3$).}
    \label{fig:loss_landscape}
    \vspace{-5pt}
\end{figure*}

\begin{wraptable}{r}{0.35\linewidth}
\vspace{-13pt}
\centering
\scriptsize
\setlength{\tabcolsep}{2pt}
\renewcommand{\arraystretch}{1.05}
\newcommand{\std}[1]{\,{\tiny$\pm$}\,\tiny #1}

\caption{Comparison of SAM-style perturbation sharpness on MNLI.} 
\label{tab:perturbation_sharpness}

\begin{tabular}{lcc}
\toprule
\textbf{Method} & \textbf{PS} & $\boldsymbol{\Delta}$\textbf{Acc} \\
\midrule
FFA-LoRA & 0.03608 \std{0.04565} & 0.981\std{1.182} \\
RoLoRA   & 0.01972 \std{0.02530} & 0.321\std{0.045} \\
AS-LoRA  & $\textbf{0.00992}$\std{0.01281} & $\textbf{0.247}$\std{0.282} \\
\bottomrule
\end{tabular}
\vspace{-5pt}
\end{wraptable}

\textbf{Sharpness and loss landscape analysis.} 
To empirically validate Theorem~\ref{thm:implicit}, which establishes that AS-LoRA's curvature-penalty score implicitly biases the trajectory toward flatter regions of the loss landscape, we measure the SAM-style perturbation sharpness
$
PS = L(w+\rho \hat g)-L(w), ~ 
\hat g = \nabla L(w)/\|\nabla L(w)\|,
$
together with the induced accuracy drop ($\Delta$Acc) \cite{foret2021sam}. The results in Table~\ref{tab:perturbation_sharpness} indicates that the solution found by AS-LoRA is more robust to local perturbations and lies in a flatter region of the loss surface. This observation is further supported by the loss landscape visualization in Figure~\ref{fig:loss_landscape}. Compared with FFA-LoRA and RoLoRA, the contour around the AS-LoRA solution appears to suggest a flatter optimization geometry. Such flatness is known to correlate with improved stability and generalization \cite{kaddour2022flat}, which is consistent with the empirical results in TABLE \ref{tab:glue_dp} and TABLE \ref{tab:squad_vision_dp}.

\begin{wraptable}{r}{0.34\textwidth}
\vspace{-13pt}
\centering
\scriptsize
\setlength{\tabcolsep}{2pt}
\renewcommand{\arraystretch}{1.05}
\newcommand{\std}[1]{\,{\tiny$\pm$}\,\tiny #1}
\caption{
    Comparison of curvature computation schedules in AS-LoRA under DP-SGD ($\epsilon=3$).
}
\label{tab:score_computation}
\begin{tabular}{lcc}
\toprule
\textbf{Method} & \textbf{MNLI (m)} & \textbf{MNLI (mm)} \\
\midrule
\multicolumn{3}{l}{\textit{Periodic computation}} \\
($f{=}5$)   & 77.99\std{1.17} & 79.39\std{1.07} \\
($f{=}10$)  & 78.83\std{1.08} & 79.76\std{0.80} \\
\midrule
\multicolumn{3}{l}{\textit{Late-phase computation}} \\
($T^*{=}80$) & 78.91\std{1.11} & 79.87\std{0.95} \\
($T^*{=}90$) & 78.85\std{0.88} & \textbf{80.11}\std{0.74} \\
\midrule
\multicolumn{3}{l}{\textit{Original}} \\
Proposed (ours) & \textbf{79.15}\std{0.23} & 80.03\std{0.40} \\
\bottomrule
\end{tabular}
\vspace{-20pt}
\end{wraptable}

\textbf{Overhead analysis.}
Under AS-LoRA with FD approximation, each client evaluates the local loss at two perturbed points 
$\mathcal{L}_k^t\!\left(W_t\pm \epsilon_{\mathrm{fd}} v_{a,k}^t(n)\right)$ per layer (Algorithm \ref{alg:score_fd}). In our setting, this incurs a $160\%$ increase in training cost over the baseline, which may be impractical for resource-constrained deployments. We therefore propose two complementary strategies that jointly reduce this overhead from $160\%$ to $\mathbf{8\%}$ while preserving the performance gains. \emph{One-sided FD.} Replacing the symmetric difference with the forward difference $\mathcal{L}_k^t\!\left(W_t+ \epsilon_{\mathrm{fd}} v_{a,k}^t(n)\right)- \mathcal{L}_k^t\!\left(W_t\right)$ requires only a single additional forward pass, halving the curvature-evaluation cost. \emph{Occasional curvature computation.} We reduce the frequency of computing the curvature part $g_a^\top H_a g_a$ in $S_a=\|g_a\|^2-\tfrac{\eta}{2}g_a^\top H_a g_a$, and uses only $\|g_a\|^2$, which is free through training, otherwise. We compare two scheduling variants: \emph{(i) Periodic computation}, inspired by the periodic SVD \cite{lee2025fedsvd}, evaluates the full score every $f \in \{5, 10\}$ rounds, and \emph{(ii) Late-phase computation} evaluates it only after round $T^* \in \{80, 90\}$. The latter is motivated by our empirical observation that $\|g_a\|^2$ dominate the optimization in early training, whereas $g_a^\top H_a g_a$ becomes informative as $\|g_a\|^2$ diminishes. As shown in Table~\ref{tab:score_computation}, late-phase computation matches the full FD-based AS-LoRA and incur only a marginal performance drop in terms of variance, causing only a $\mathbf{8\%}$ computation increase at $T^* = 90$ ($20\times$ reduction from the original). These results suggest that full per-round curvature scoring is unnecessary, and that selective computation, at the last phase in particular, can substantially reduce overhead. Communication overhead is negligible: AS-LoRA introduces only $N$ scalar scores (uplink) and $N$ mode bits (downlink) per round, accounting for less than $0.02\%$ of the active LoRA parameter traffic in our $24$-layer configuration. Detailed analyses are provided in Appendix~\ref{app:communication_overhead} and Appendix~\ref{app:computational_overhead}.

\textbf{Effect of gradient smoothing under DP-SGD.} 
Motivated by prior work~\cite{liu2026rethinking, liang2024differentially, wang2020dp} demonstrating that applying smoothing after Gaussian noise injection in DP-SGD improves utility, we investigate several smoothing strategies such as Gaussian, Laplacian, and EMA applied to our method. Detailed results are provided in Appendix~\ref{app:kernel}.

\vspace{-3mm}

\section{Conclusion} \label{sec:conclusion}

In this paper, we proposed AS-LoRA, an adaptive optimization framework for LoRA in differentially private FL. 
Different to the prior work on federated LoRA whose optimization is fixed or layer-tied, AS-LoRA introduces a layer-wise adaptive selection mechanism that dynamically determines which component to update based on a theoretically grounded score function. 
We provided theoretical analysis showing that AS-LoRA eliminates the irreducible reconstruction error floor of static methods and achieves accelerated convergence with implicit flatter solution, without incurring additional privacy cost. Empirically, AS-LoRA consistently outperforms prior methods across language and vision tasks under non-IID and differentially private settings.

\textbf{Limitations and future works.}
Our current analysis and algorithm focus on layer-wise greedy selection. While this design enables tractable analysis and strong empirical performance, it does not explicitly account for possible dependencies across layers. In practice, the mode selection dynamics in Appendix~\ref{app:mode_selection} suggest that, particularly for language tasks, cross-layer structure may exist in the selection process. Extending AS-LoRA to a joint layer-wise decision framework that models such inter-layer correlations is an important direction for future work.

\textbf{Broader impacts.}
The proposed method advances differentially private FL, which has clear positive societal impact by enabling collaborative model training without exposing private user data—relevant. Potential negative impacts are limited: the method is an optimization technique for fine-tuning foundation models and does not directly enable new harmful capabilities. Improving DP fine-tuning efficiency could in principle lower barriers to deploying powerful models, but this is a general concern shared by all efficiency improvements in machine learning and not specific to this work.
\bibliographystyle{unsrt}  
\bibliography{references}  
\clearpage
\appendix

\section{Related Works}
\label{sec:Appendix A}

\subsection{Federated learning with LoRA}
Recent studies have explored integrating LoRA into FL to alleviate the communication and computation overhead associated with fine-tuning large-scale models \cite{kuo2024federated, cho2024heterogeneous, fang2024automated, yi2023pfedlora}. 
FedIT \cite{zhang2024towards} is one of the earliest attempts to incorporate LoRA into FL, demonstrating that low-rank adapters can effectively reduce communication costs while maintaining competitive performance.
However, FedIT aggregates the two low-rank components in a straightforward manner and does not explicitly account for aggregation errors that arise during the aggregation process. 
To mitigate aggregation mismatch caused by heterogeneous client updates, several extensions have been proposed.
FLoRA \cite{wang2024flora} addresses heterogeneity and aggregation errors by stacking client-specific LoRA adapters at the server for aggregation.
FedEx-LoRA \cite{singhal2024fedex} alleviates aggregation error by introducing a residual correction term added to a fixed pre-trained weight matrix.
FlexLoRA \cite{bai2024federated} further addresses heterogeneity by reconstructing an ideal weight update matrix from the aggregated client adapters at the server, followed by singular value decomposition (SVD) to redistribute LoRA adapters with different ranks to individual clients.

\subsection{Differentially private federated LoRA}
More recent studies have investigated federated LoRA under differential privacy (DP) constraints.
FFA-LoRA \cite{sun2024improving} provides a detailed analysis of the noise amplification problem that arises when DP is applied to federated LoRA training, and mitigates this issue by updating only the $B$ component.
However, restricting updates to the $B$ component leads to a loss of expressiveness in LoRA.
To address this limitation, FedSVD \cite{lee2025fedsvd} similarly updates only the $B$ component, while leveraging orthogonalization and low-rank decomposition to alleviate the expressiveness loss and improve training stability under DP constraints.
Another line of work, which is most closely related to our study, aims to alleviate aggregation error by adopting alternating optimization strategies.
RoLoRA \cite{chen2025robust} and LoRA-A$^2$ \cite{koo2025towards} alternate the optimization of the two LoRA components across communication rounds, updating only one component at a time to reduce aggregation mismatch.
La-LoRA \cite{liu2026rethinking} further extend alternating optimization from communication rounds to local update steps.
While these methods empirically improve training stability, they rely on fixed and symmetric alternating schedules and do not explicitly account for the asymmetric roles of the two LoRA components during training.
In summary, existing federated LoRA-based approaches either restrict updates to a single component or adopt fixed, symmetric alternating schedules, both of which overlook the structural asymmetry between the two LoRA components and the round-wise dynamics of training. This motivates AS-LoRA, which makes principled, adaptive per-layer decisions based on training dynamics.

\newpage
\section{AS-LoRA Algorithm}
\label{sec:Appendix B}

\begin{figure}[h]
    \centering
    \includegraphics[width=1.0\textwidth]{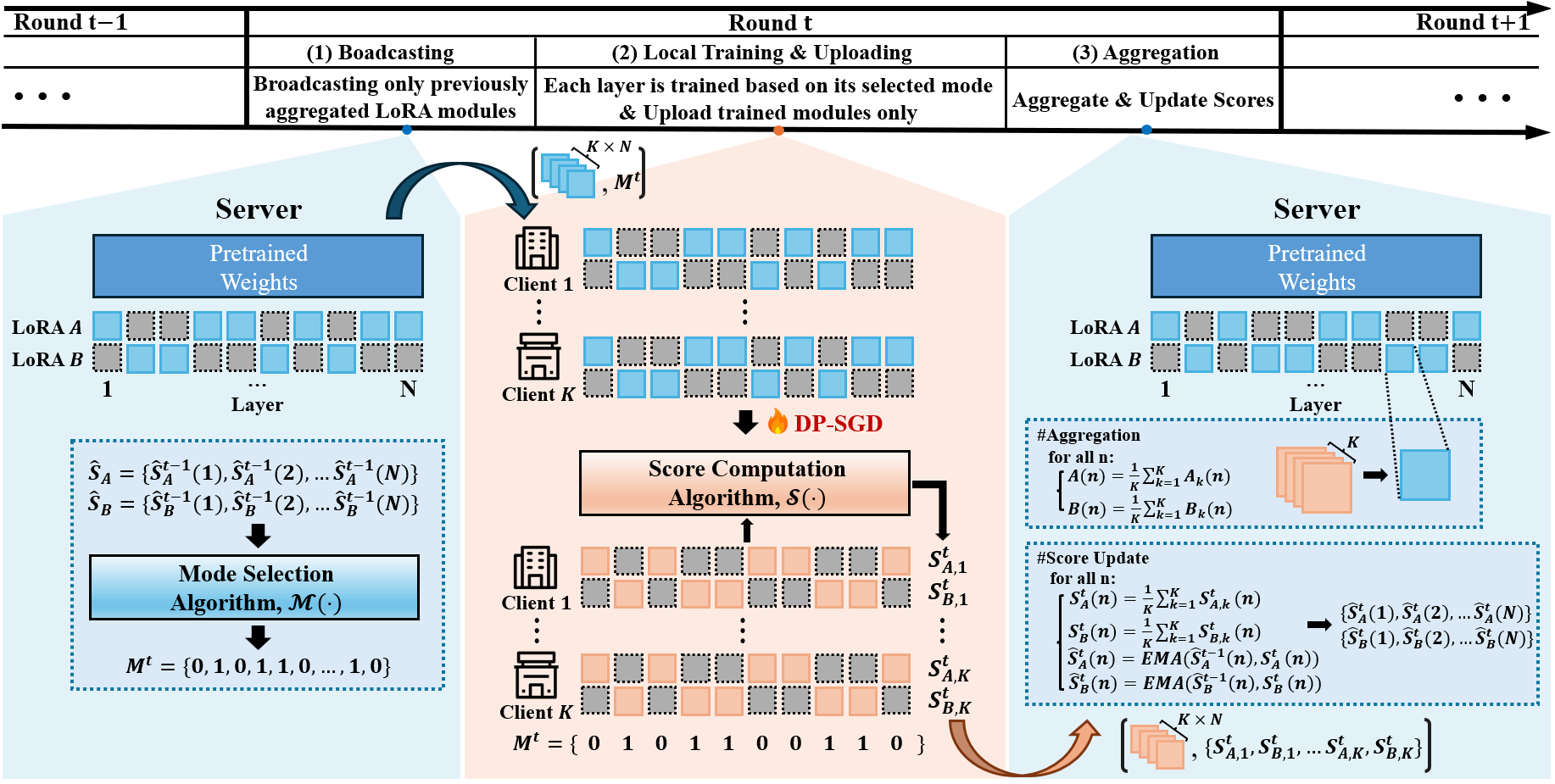}
    \caption{Detail of the proposed AS-LoRA framework. }
    \label{fig:AS-LoRA_detail}
\end{figure}

\begin{figure}[h]
    \centering
    \includegraphics[width=1.0\textwidth]{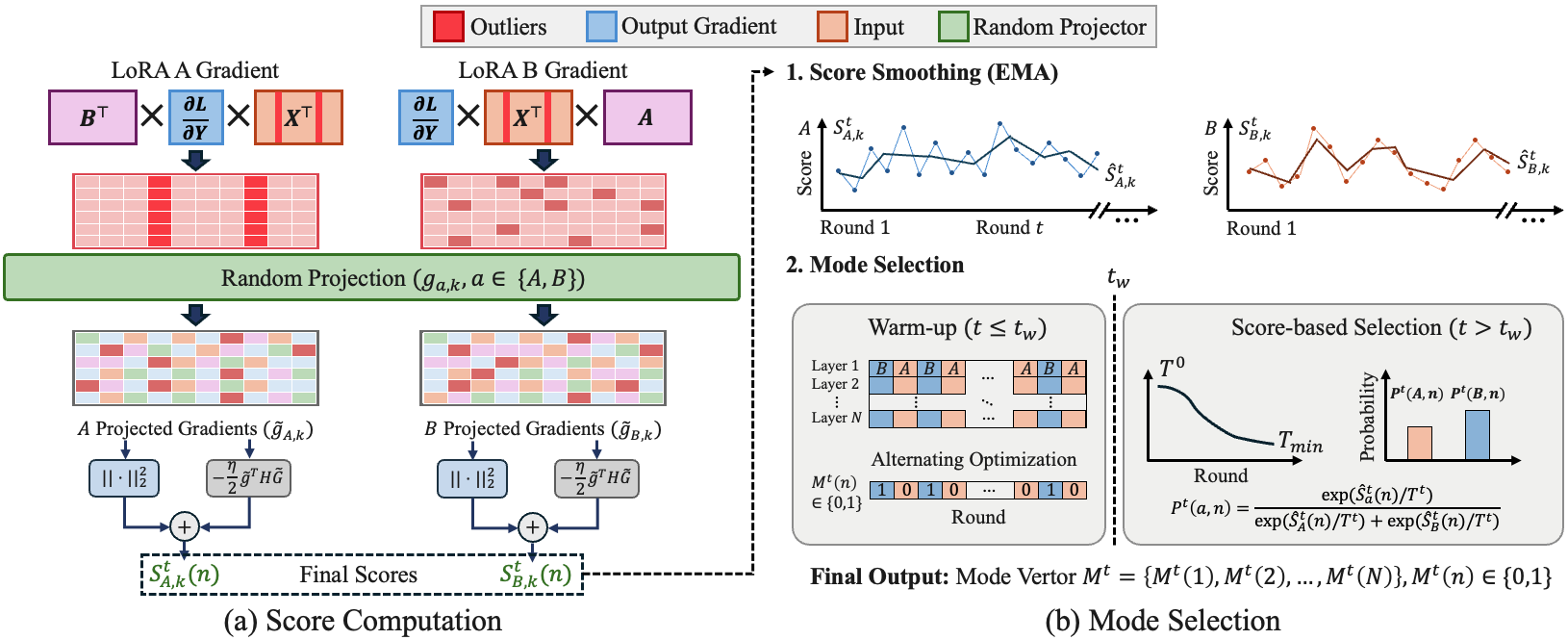}
    \caption{ Detailed visualization of the proposed AS-LoRA score computation and mode selection procedures. (a) Component-wise scores are computed from the projected gradients of LoRA $A$ and $B$ matrices. Random projection mitigates the effect of channel-wise gradient outliers before applying the squared-gradient and curvature-aware score terms. (b) The server smooths the aggregated scores using EMA and performs layer-wise mode selection. During warm-up, AS-LoRA alternates between components, while after warm-up it uses a temperature-scaled score-based selection rule to produce the final mode vector. }
    \label{fig:score_mode_detail}
\end{figure}

\FloatBarrier
\begin{algorithm}[h]
\caption{AS-LoRA}
\label{alg:as_lora}
\begin{algorithmic}[1]

\STATE \textbf{Input:} communication rounds $T$, clients $K$, LoRA layers $N$
\STATE Initialize $A^0(n), B^0(n)$ and server scores 
$\hat{S}_A^{0}(n), \hat{S}_B^{0}(n)$ for all $n$

\FOR{$t=1$ \TO $T$}
    \STATE $M^t \leftarrow \mathcal{M}(\hat{S}_A^{t-1}, \hat{S}_B^{t-1})$
    
    \FOR{$k=1$ \TO $K$ \textbf{in parallel}}
        \STATE Update LoRA parameters using DP-SGD according to $M^t$
        \STATE Compute local scores 
        $S_{A,k}^t(n), S_{B,k}^t(n) \leftarrow \mathcal{S}(\cdot)$
        \STATE Upload updated parameters and scores to the server
    \ENDFOR

    \FOR{$n=1$ \TO $N$}
        \IF{$M^t(n)=0$}
            \STATE $A^t(n) \leftarrow \frac{1}{K}\sum_{k=1}^{K} A_k^t(n)$, 
            \quad $B^t(n) \leftarrow B^{t-1}(n)$
        \ELSE
            \STATE $B^t(n) \leftarrow \frac{1}{K}\sum_{k=1}^{K} B_k^t(n)$, 
            \quad $A^t(n) \leftarrow A^{t-1}(n)$
        \ENDIF

        \STATE $S_A^t(n) \leftarrow \frac{1}{K}\sum_{k=1}^{K} S_{A,k}^t(n)$, 
        \quad
        $S_B^t(n) \leftarrow \frac{1}{K}\sum_{k=1}^{K} S_{B,k}^t(n)$

        \STATE $\hat{S}_A^t(n) \leftarrow 
        \mathrm{EMA}\left(\hat{S}_A^{t-1}(n), S_A^t(n)\right)$,
        \quad
        $\hat{S}_B^t(n) \leftarrow 
        \mathrm{EMA}\left(\hat{S}_B^{t-1}(n), S_B^t(n)\right)$
    \ENDFOR
\ENDFOR

\STATE \textbf{Output:} final LoRA parameters $\{A^T(n), B^T(n)\}_{n=1}^{N}$

\end{algorithmic}
\end{algorithm}
\begin{algorithm}[h]
\caption{Score Computation $\mathcal{S}_{\|{g}\|^2}(\cdot)$ with Random Projection}
\label{alg:score_gnorm_rp}
\begin{algorithmic}[1]

\STATE \textbf{Input:} local loss $\mathcal{L}_k^t$, LoRA parameters $\{A_k^t(n), B_k^t(n)\}_{n=1}^{N}$, projection matrices $\{R_A(n), R_B(n)\}_{n=1}^{N}$
\STATE \textbf{Output:} local scores $\{S_{A,k}^t(n), S_{B,k}^t(n)\}_{n=1}^{N}$

\FOR{$n=1$ \TO $N$}
    \STATE Compute gradients:
    \[
    g_{A,k}^t(n) \leftarrow \nabla_{A_k^t(n)} \mathcal{L}_k^t,
    \quad
    g_{B,k}^t(n) \leftarrow \nabla_{B_k^t(n)} \mathcal{L}_k^t
    \]

    \STATE Apply Gaussian random projection:
    \[
    \widetilde{g}_{A,k}^t(n) \leftarrow g_{A,k}^t(n) R_A(n),
    \quad
    \widetilde{g}_{B,k}^t(n) \leftarrow R_B(n) g_{B,k}^t(n)
    \]

    \STATE Compute projected squared-gradient scores:
    \[
    S_{A,k}^t(n) \leftarrow \left\|\widetilde{g}_{A,k}^t(n)\right\|_2^2
    \]
    \[
    S_{B,k}^t(n) \leftarrow \left\|\widetilde{g}_{B,k}^t(n)\right\|_2^2
    \]
\ENDFOR

\RETURN $\{S_{A,k}^t(n), S_{B,k}^t(n)\}_{n=1}^{N}$

\end{algorithmic}
\end{algorithm}

\begin{algorithm}[h]
\caption{Score Computation $\mathcal{S}_{\mathrm{HVP}}(\cdot)$ with Random Projection}
\label{alg:score_hvp}
\begin{algorithmic}[1]

\STATE \textbf{Input:} local loss $\mathcal{L}_k^t$, LoRA parameters $\{A_k^t(n), B_k^t(n)\}_{n=1}^{N}$, projection matrices $\{R_A(n), R_B(n)\}_{n=1}^{N}$, learning rate $\eta$
\STATE \textbf{Output:} local scores $\{S_{A,k}^t(n), S_{B,k}^t(n)\}_{n=1}^{N}$

\FOR{$n=1$ \TO $N$}

    \STATE Compute gradients:
    \[
    g_{A,k}^t(n) \leftarrow \nabla_{A_k^t(n)} \mathcal{L}_k^t,
    \quad
    g_{B,k}^t(n) \leftarrow \nabla_{B_k^t(n)} \mathcal{L}_k^t
    \]

    \STATE Apply random projection:
    \[
    \widetilde{g}_{A,k}^t(n) \leftarrow g_{A,k}^t(n) R_A(n),
    \quad
    \widetilde{g}_{B,k}^t(n) \leftarrow R_B(n) g_{B,k}^t(n)
    \]

    \STATE Compute Hessian-vector products:
    \[
    h_{A,k}^t(n) \leftarrow H_{A(n)}(W_t)\,\widetilde{g}_{A,k}^t(n),
    \quad
    h_{B,k}^t(n) \leftarrow H_{B(n)}(W_t)\,\widetilde{g}_{B,k}^t(n)
    \]

    \STATE Compute scores:
    \[
    S_{A,k}^t(n) \leftarrow 
    \|\widetilde{g}_{A,k}^t(n)\|_2^2
    -
    \frac{\eta}{2}
    \widetilde{g}_{A,k}^t(n)^\top h_{A,k}^t(n)
    \]

    \[
    S_{B,k}^t(n) \leftarrow 
    \|\widetilde{g}_{B,k}^t(n)\|_2^2
    -
    \frac{\eta}{2}
    \widetilde{g}_{B,k}^t(n)^\top h_{B,k}^t(n)
    \]

\ENDFOR

\RETURN $\{S_{A,k}^t(n), S_{B,k}^t(n)\}_{n=1}^{N}$

\end{algorithmic}
\end{algorithm}

\begin{algorithm}[h]
\caption{Score Computation $\mathcal{S}_{\mathrm{FD}}(\cdot)$ with Random Projection}
\label{alg:score_fd}
\begin{algorithmic}[1]

\STATE \textbf{Input:} local loss $\mathcal{L}_k^t$, LoRA parameters $\{A_k^t(n), B_k^t(n)\}_{n=1}^{N}$, projection matrices $\{R_A(n), R_B(n)\}_{n=1}^{N}$, learning rate $\eta$, finite-difference radius $\epsilon_{\mathrm{fd}}$
\STATE \textbf{Output:} local scores $\{S_{A,k}^t(n), S_{B,k}^t(n)\}_{n=1}^{N}$

\FOR{$n=1$ \TO $N$}

    \STATE Compute gradients:
    \[
    g_{A,k}^t(n) \leftarrow \nabla_{A_k^t(n)} \mathcal{L}_k^t,
    \quad
    g_{B,k}^t(n) \leftarrow \nabla_{B_k^t(n)} \mathcal{L}_k^t
    \]

    \STATE Apply random projection:
    \[
    \widetilde{g}_{A,k}^t(n) \leftarrow g_{A,k}^t(n)R_A(n),
    \quad
    \widetilde{g}_{B,k}^t(n) \leftarrow R_B(n)g_{B,k}^t(n)
    \]

    \FOR{$a \in \{A,B\}$}

        \STATE Normalize the projected gradient direction:
        \[
        v_{a,k}^t(n) \leftarrow 
        \frac{\widetilde{g}_{a,k}^t(n)}
        {\|\widetilde{g}_{a,k}^t(n)\|_2+\epsilon}
        \]

        \STATE Estimate curvature by finite difference:
        \[
        c_{a,k}^t(n) \leftarrow
        \frac{
        \mathcal{L}_k^t\!\left(W_t+\epsilon_{\mathrm{fd}} v_{a,k}^t(n)\right)
        -2\mathcal{L}_k^t(W_t)
        +\mathcal{L}_k^t\!\left(W_t-\epsilon_{\mathrm{fd}} v_{a,k}^t(n)\right)
        }
        {\epsilon_{\mathrm{fd}}^2}
        \]

        \STATE Compute finite-difference score:
        \[
        S_{a,k}^t(n) \leftarrow
        \|\widetilde{g}_{a,k}^t(n)\|_2^2
        -
        \frac{\eta}{2}
        c_{a,k}^t(n)
        \|\widetilde{g}_{a,k}^t(n)\|_2^2
        \]

    \ENDFOR

\ENDFOR

\RETURN $\{S_{A,k}^t(n), S_{B,k}^t(n)\}_{n=1}^{N}$

\end{algorithmic}
\end{algorithm}
\begin{algorithm}[h]
\caption{Mode Selection $\mathcal{M}(\cdot)$ }
\label{alg:mode_selection}
\begin{algorithmic}[1]

\STATE \textbf{Input:} EMA scores $\{\hat{S}_A^t(n), \hat{S}_B^t(n)\}_{n=1}^{N}$, round $t$, warm-up period $t_w$, temperature parameters $(T^0, T_{\min}, \gamma)$
\STATE \textbf{Output:} mode vector $M^t = \{M^t(n)\}_{n=1}^{N}$

\IF{$t \le t_w$}
    \STATE \textbf{// warm-up: alternating schedule}
    \FOR{$n=1$ \TO $N$}
        \STATE $M^t(n) \leftarrow (t \bmod 2)$
    \ENDFOR

\ELSE
    \STATE Compute temperature:
    \[
    T^t \leftarrow \max\!\left(T_{\min},\, T^0 \cdot \gamma^{(t-t_w)}\right)
    \]

    \FOR{$n=1$ \TO $N$}

        \STATE Compute selection probabilities:
        \[
        P^t(A,n) \leftarrow 
        \frac{\exp\!\left(\hat{S}_A^t(n)/T^t\right)}
        {\exp\!\left(\hat{S}_A^t(n)/T^t\right) + \exp\!\left(\hat{S}_B^t(n)/T^t\right)}
        \]

        \STATE Sample mode:
        \[
        M^t(n) \sim \mathrm{Bernoulli}(P^t(A,n))
        \]

    \ENDFOR
\ENDIF

\RETURN $M^t$

\end{algorithmic}
\end{algorithm}
\FloatBarrier

\subsection{Visualization of AS-LoRA framework}
Fig~\ref{fig:AS-LoRA_detail} illustrates the overall workflow of the proposed AS-LoRA framework. 
At each communication round, the server first determines the layer-wise mode based on historical scores and broadcasts only the selected LoRA components. 
Each client then performs local updates using DP-SGD according to the assigned modes, computes component-wise scores, and uploads both the updated parameters and scores to the server. 
The server aggregates the selected components and updates the global scores, which are used for the next round of mode selection. 
This process enables adaptive, layer-wise optimization across communication rounds while mitigating aggregation error and unnecessary noise amplification. The overall training procedure of AS-LoRA is summarized in Algorithm~\ref{alg:as_lora}.

\subsection{Score Computation and Mode Selection}
Figure~\ref{fig:score_mode_detail} provides a detailed illustration of the score 
computation and mode selection procedures in AS-LoRA.

\textbf{Score Computation.}
Figure~\ref{fig:score_mode_detail}(a) depicts the score computation procedure 
performed on each client.
As illustrated, the gradient of LoRA-$A$ exhibits channel-wise outliers, where 
the gradient magnitude is highly concentrated in a small subset of channels.
To mitigate this, a Gaussian random projection is applied to each gradient, 
redistributing the concentrated magnitude across dimensions and enabling more 
reliable score estimation.
The score computation procedures with random projection are described in 
Algorithms~\ref{alg:score_gnorm_rp}--\ref{alg:score_fd}.

\textbf{Mode Selection.}
Figure~\ref{fig:score_mode_detail}(b) illustrates the mode selection procedure.
The server-side scores are smoothed via EMA to reduce round-to-round variance, 
and the mode for each layer is determined based on the smoothed scores.
During the warm-up phase ($t \leq t_w$), all layers follow a simple alternating 
schedule, where the $A$ and $B$ components are activated in alternation across 
rounds, and the mode vector $M^t(n) \in \{0, 1\}$ changes accordingly.
After the warm-up phase ($t > t_w$), score-based selection is activated.
The selection probability for each component is computed via a temperature-scaled 
softmax.

Here, the temperature $T^t$ follows an annealing schedule that decays from an 
initial value $T^0$ to a minimum $T_{\min}$ over training rounds.
This annealing encourages exploration of both components in early rounds while 
gradually shifting toward exploitation of the higher-scoring component as 
training progresses.
The final output is the mode vector
$M^t = \{M^t(1), M^t(2), \ldots, M^t(N)\}$, $M^t(n) \in \{0, 1\}$,
which specifies the active LoRA component for each layer at round $t$.
The full mode selection procedure is summarized in 
Algorithm~\ref{alg:mode_selection}.

\section{Preliminaries for Theoretical Analysis}
\label{app:prelim}

\subsection{Background}
We analyze AS-LoRA along two axes: \emph{(i) layer-wise freedom}---the ability to activate $A^{(n)}$ or $B^{(n)}$ independently per layer---and \emph{(ii) score-driven selection}---the greedy rule $\arg\max_{a\in\{A,B\}} S_a^{(n)}$ with $S_a^{(n)}=\|g_a^{(n)}\|^2-\tfrac{\eta}{2}(g_a^{(n)})^\top H_a^{(n)} g_a^{(n)}$. Each axis contributes a distinct and quantifiable advantage over the layer-tied or fixed-schedule baselines (FFA-LoRA~\citep{sun2024improving}, RoLoRA~\citep{chen2025robust}). The analysis combines block-coordinate descent theory~\cite{nesterov2012bcd}, Polyak--{\L}ojasiewicz convergence~\citep{karimi2016pl}, and DP-SGD privacy accounting~\citep{abadi2016deep,mironov2017renyi,wang2019sgm}.

\subsection{Setup}

Let $\Loss(W)=\tfrac{1}{K}\sum_{k=1}^K\Loss_k(W)$ be the federated loss over $K$ clients with per-client dataset size $m$ (total $Km$ samples). For layer $n\in[N]$, the frozen weight is $W_0^{(n)}\in\R^{d\times d}$ and the LoRA update factorizes as $B^{(n)}A^{(n)}$ with $B^{(n)}\in\R^{d\times r}$, $A^{(n)}\in\R^{r\times d}$, and rank $r\!\ll\!d$, so that the effective weight is $W^{(n)}=W_0^{(n)}+B^{(n)}A^{(n)}$ acting on an input $X\in\R^{d\times m}$ as $W^{(n)}X^{(n)}=W_0^{(n)}X+B^{(n)}A^{(n)}X^{(n)}$

\paragraph{Stochastic gradient and population Hessian.}
For each block label $a\in\{A,B\}$, let $g_a^{(n)}$ denote the \emph{mini-batch stochastic gradient} of $\Loss$ with respect to the factor matrix $a^{(n)}$, satisfying $\E[g_a^{(n)}]=\nabla_{a^{(n)}}\Loss$ over the random sampling. Concretely, $g_A^{(n)}\in\R^{r\times d}$ and $g_B^{(n)}\in\R^{d\times r}$ are matrix-valued. The corresponding \emph{deterministic population Hessian} $H_a^{(n)}(W):=\nabla^2_{a^{(n)}}\Loss(W)$ acts on the vectorized factor and is treated as a $(\dim a^{(n)})\!\times\!(\dim a^{(n)})$ matrix throughout (with $\dim A^{(n)}=\dim B^{(n)}=rd$). Randomness enters only through $g_a^{(n)}$ and DP-SGD noise $\Gauss_a\!\sim\!\Gauss(0,\sigma^2 C^2 I)$ (clipping bound $C$), and $\nabla_{a^{(n)}}\Loss$ and $H_a^{(n)}(W)$ are deterministic functions of $W$ as in the standard convention for second-order SGD/BCD analyses~\cite{karimi2016pl,nesterov2012bcd}. With $\|\cdot\|$ denoting the Frobenius norm for matrices and the Euclidean norm for vector objects, the per-layer block Rayleigh quotient is $\lambda_a^{(n)}:=\langle g_a^{(n)}, H_a^{(n)} g_a^{(n)}\rangle/\|g_a^{(n)}\|^2$ and the per-layer block-greedy score is $S_a^{(n)}:=\|g_a^{(n)}\|^2-\tfrac{\eta}{2}\langle g_a^{(n)}, H_a^{(n)} g_a^{(n)}\rangle$. 

\begin{remark}[Practical estimation of $H_a^{(n)} g_a^{(n)}$]\label{rem:hvp}
All theoretical statements use the population Hessian $H_a^{(n)}(W)=\nabla^2_{a^{(n)}}\Loss(W)$. In practice, the curvature term $\langle g_a^{(n)}, H_a^{(n)} g_a^{(n)}\rangle$ is computed via a single Hessian-vector product on the same mini-batch used for $g_a^{(n)}$ (Pearlmutter trick) or via a Hutchinson-style sketch~\citep{yao2020pyhessian}; the resulting estimation bias is $O(\sigma_g^2/m)$ and is absorbed into the $\eta\beta\sigma_g^2$ noise term of Theorem~\ref{thm:conv} without altering any rate.
\end{remark}

\paragraph{Moore--Penrose pseudoinverse and orthogonal projectors.}
Several reconstruction-floor and subspace-misalignment quantities are stated in terms of the Moore--Penrose pseudoinverse, which we briefly recall. For any matrix $M\in\R^{p\times q}$, the \emph{Moore--Penrose pseudoinverse} $M^\dagger\in\R^{q\times p}$ is the unique matrix satisfying the four Penrose conditions $M M^\dagger M = M$, $M^\dagger M M^\dagger = M^\dagger$, $(M M^\dagger)^\top = M M^\dagger$, $(M^\dagger M)^\top = M^\dagger M$~\citep{penrose1955}. Three properties are used repeatedly. (i) When $M$ is square and invertible, $M^\dagger = M^{-1}$. (ii) When $M$ is row full-rank (the generic case for our LoRA factors $A^t(n)\in\R^{r\times d}$ with $r<d$), $M^\dagger = M^\top(MM^\top)^{-1}$, and $A^t(n)\,A^t(n)^\dagger = I_r$ acts as a right inverse. (iii) The two products $M M^\dagger\in\R^{p\times p}$ and $M^\dagger M\in\R^{q\times q}$ are precisely the orthogonal projectors onto $\mathrm{col}(M)\subseteq\R^p$ and $\mathrm{row}(M)\subseteq\R^q$ respectively. Operationally, the least-squares minimum-norm solution of $Mx=b$ is $x^\star = M^\dagger b$, which is the property used to obtain the closed-form FFA-LoRA minimizer in the proof of Theorem~\ref{thm:irreducible}. All projectors $P_{A^t(n)}=A^t(n)^\dagger A^t(n)$ used in the analysis are of form (iii) and project onto the row-span of the corresponding factor matrix.

\subsection{Assumptions}
First, we collect here all assumptions invoked by the main text. Assumptions~\ref{asm:smooth}--\ref{asm:pl} are the standard block / per-layer counterparts of the smoothness, bounded-variance, Hessian-spectrum, and Polyak--{\L}ojasiewicz (PL) conditions used in modern SGD/BCD/PL analyses~\cite{nesterov2012bcd,karimi2016pl}. Strong convexity of the Assumption \ref{asm:hess} is invoked only on the per-layer reconstruction quadratic, not on the full non-convex loss, and the PL condition of Assumption \ref{asm:pl} replaces it for end-to-end statements.

\begin{assumption}[Block smoothness]\label{asm:smooth}
Each $\Loss_k$ is $L$-smooth and block-smooth with constants $L_A^{(n)},L_B^{(n)}$, and $\beta:=\max_{a,n}L_a^{(n)}$.
\end{assumption}
\begin{assumption}[Bounded stochastic variance]\label{asm:var}
$\E\|g_a^{(n)}-\nabla_{a(n)}\Loss\|^2\le\sigma_g^2$ for every layer $n$ and block label $a\in\{A,B\}$, and per-sample clipping $\|g_{a,i}^{(n)}\|\le C$ a.s.
\end{assumption}
\begin{assumption}[Block strong convexity and Lipschitz Hessian]\label{asm:hess}
$\mu_a^{(n)}I\preceq H_a^{(n)}\preceq L_a^{(n)} I$, and $H$ is $L_H$-Lipschitz.
\end{assumption}
\begin{assumption}[PL condition, non-convex]\label{asm:pl}
$\|\nabla\Loss\|^2\ge 2\mu(\Loss-\Loss^\star)$ for some $\mu>0$~\citep{karimi2016pl}.
\end{assumption}

\subsection{Mathematical Preliminaries}

This appendix collects the standard tools from matrix analysis \cite{horn2012matrix}, high-dimensional probability \cite{vershynin2018}, and convex optimization \cite{nesterov2018book}.

\begin{lemma}[Norm submultiplicativity and SVD-based reverse bound]\label{lem:norm-sandwich}
For any $M\in\R^{p\times d}$ and $X\in\R^{d\times m}$ of compatible dimensions (so that $MX\in\R^{p\times m}$ is well defined),
\begin{equation}
\|MX\|_F \;\le\; \sigma_{\max}(X)\,\|M\|_F,
\end{equation}
where $\sigma_{\max}(X)=\|X\|_\text{op}$ is the largest singular value of $X$. We refer to this as the standard inequality. If the rows of $M$ lie in $\mathrm{col}(X^\top)=\mathrm{row}(X)$, equivalently $M=M\,P_{X}$ where $P_X=X^\top(XX^\top)^{-1}X$ when $X$ has full row rank, then the matching reverse inequality holds:
\begin{equation}
\|MX\|_F \;\ge\; \sigma_{\min}^+(X)\,\|M\|_F,
\end{equation}
with $\sigma_{\min}^+(X)$ the smallest \emph{nonzero} singular value of $X$.
\end{lemma}
\begin{proof}

The upper bound is directly from \cite{horn2012matrix}. For the proof of lower bound, let $X=U\Sigma V^\top$ be the SVD with $\Sigma=\mathrm{diag}(\sigma_1,\ldots,\sigma_k,0,\ldots,0)$, $\sigma_1\ge\cdots\ge\sigma_k>0$, $k=\mathrm{rank}(X)$. Then $\|MX\|_F^2=\|MU\Sigma V^\top\|_F^2=\|MU\Sigma\|_F^2$ (Frobenius is invariant under right multiplication by orthogonal $V^\top$). Writing $MU=:N$ with columns $n_1,\ldots,n_d$, $\|MU\Sigma\|_F^2=\sum_{i=1}^d\sigma_i^2\|n_i\|^2$. If rows of $M$ lie in $\mathrm{row}(X)$, then $n_i=0$ for all $i>k$, so $\sum_i\sigma_i^2\|n_i\|^2\ge\sigma_k^2\sum_{i\le k}\|n_i\|^2=(\sigma_{\min}^+(X))^2\|M\|_F^2$.
\end{proof}

\begin{lemma}[High-probability bound of covariance estimation \cite{vershynin2018}, Rem.~4.7.3]\label{lem:cov-conc}
Let $x_1,\ldots,x_n\in\R^d$ be i.i.d.\ zero-mean random vectors with population covariance $\Sigma:=\E[x_ix_i^\top]$, satisfying 
\begin{equation}\label{cond_shapeaware}
\bigl\|\langle X,v\rangle\bigr\|_{\psi_2} \;\le\; \kappa\,\bigl\|\langle X,v\rangle\bigr\|_{L^2} \qquad\text{for every } v\in\R^d,
\end{equation}
with parameter $\kappa$. Then for every $u\ge 0$, with probability at least $1-2e^{-u}$,
\begin{equation}
\Bigl\|\tfrac{1}{n}\!\sum_{i=1}^n\! x_ix_i^\top - \Sigma\Bigr\|_\text{op} \;\le\; C\,\kappa^2\,\|\Sigma\|_\text{op}\!\Bigl(\sqrt{(d+u)/n}+(d+u)/n\Bigr),
\end{equation}
for an absolute constant $C>0$. In particular, for $n\gtrsim d$ and any fixed confidence level $1-2e^{-u}$,
\begin{equation}
\bigl\|\tfrac{1}{n}\!\sum_{i=1}^n\! x_ix_i^\top - \Sigma\bigr\|_\text{op} \;=\; O_p\!\bigl(\kappa^2\|\Sigma\|_\text{op}\sqrt{d/n}\bigr).
\end{equation}
\end{lemma}


\begin{lemma}[Linear convergence of GD on strongly-convex smooth functions \cite{nesterov2018book}, Thm. 2.3.4]\label{lem:contract}
Let $f:\R^{p\times q}\to\R$ be twice differentiable with Hessian satisfying $\mu I\preceq\nabla^2 f\preceq L I$ for $0<\mu\le L$. Let $X^\star=\arg\min f$. Then gradient descent $X_{t+1}=X_t-\eta\nabla f(X_t)$ with step size $\eta\le 1/L$ satisfies
\begin{equation}
\|X_{t+1}-X^\star\|_F^2 \;\le\; (1-\eta\mu)\,\|X_t-X^\star\|_F^2 \;\le\; (1-\mu/L)\,\|X_t-X^\star\|_F^2.
\end{equation}
\end{lemma}

\begin{lemma}[Second-order Taylor with cubic remainder \cite{nesterov2018book}, Lemma 1.2.4]\label{lem:taylor}
Let $\Loss:\R^d\to\R$ be twice differentiable with $L_H$-Lipschitz Hessian: $\|\nabla^2\Loss(x)-\nabla^2\Loss(y)\|_\text{op}\le L_H\|x-y\|$. Then for any $W,h$,
\begin{equation}
\Bigl|\Loss(W+h)-\Loss(W)-\langle\nabla\Loss(W),h\rangle-\tfrac12 h^\top\nabla^2\Loss(W)h\Bigr| \;\le\; \tfrac{L_H}{6}\|h\|^3.
\end{equation}
\end{lemma}

\section{Proof of Theorem \ref{thm:privacy}}
\label{app:proof_of_privacy}
The DP-SGD mechanism of \cite{abadi2016deep} uses subsampled Gaussian noise: at each round, sample a fraction $q$ of clients, compute clipped gradients $\tilde g_i=\Pi_{\|\cdot\|\le C}(g_i)$, aggregate $\bar g=\sum_{i\in\text{batch}}\tilde g_i+\Gauss(0,\sigma^2 C^2 I)$. By the moments accountant / R\'enyi-DP analysis of \cite{mironov2017renyi,wang2019sgm}, each round contributes at most $\tfrac{q^2C^2}{\sigma^2}$ to the cumulative R\'enyi-DP $(\alpha,\epsilon_\alpha)$ budget at order $\alpha$. R\'enyi-DP composes additively. $T$ rounds give cumulative R\'enyi-DP cost at most $T\cdot q^2C^2/\sigma^2$ at order $\alpha$. Converting to $(\epsilon,\delta)$-DP via the standard tail bound $\epsilon=\epsilon_\alpha+\log(1/\delta)/(\alpha-1)$ and optimising over $\alpha$ yields $\epsilon\le q^2 C^2 T\log(1/\delta)/\sigma^2$. The AS-LoRA selection $\arg\max_{a\in\{A,B\}}S_a^{(n),t}$ uses the per-block scores $S_a^{(n),t}=\|\bar g_a^{(n)}\|^2-\tfrac{\eta}{2}\langle \bar g_a^{(n)}, H_a^{(n)}\bar g_a^{(n)}\rangle$, where $\bar g_a^{(n)}$ is the already-clipped, already-noised, before being aggregated. Since $H_a^{(n)}$ is also computed on the same mini-batch, the score is a deterministic function of $\bar g_a^{(n)}$. The post-processing invariance of differential privacy~\cite{dwork2014dpfoundations} says that \emph{any function applied to the output of an $(\epsilon,\delta)$-DP mechanism remains $(\epsilon,\delta)$-DP}. Since $M^t(n)$ is computed deterministically from $\{\bar g_a^{(n)}\}_{a,n}$, which are already DP-protected, no additional privacy cost is incurred. The total AS-LoRA privacy ledger therefore matches FFA-LoRA / RoLoRA exactly. \hfill$\square$

\section{Proof of Theorem \ref{thm:irreducible}}
\label{app:proof_of_irreducible}
Fix layer $n$. Under FFA-LoRA, $A^{(n)}$ is permanently fixed at $A^0(n)$ and only $B^{(n)}$ is updated by SGD with step size $\eta$. Because the per-layer reconstruction loss is convex quadratic in $B^{(n)}$ (when $A^{(n)}$, $X_n$ are fixed), standard SGD theory on convex quadratics guarantees that the iterates $B^{(n)}_t$ converge in expectation to the static minimizer $B_\text{FFA}(n)$ of
\begin{equation}
\ell^{(n)}(B) \;=\; \tfrac12\bigl\|B^\star(n)A^\star(n) X_n - B\,A^0(n) X_n\bigr\|_F^2,
\end{equation}
up to an $O(\eta\sigma_g^2)$ noise floor \cite{bottou2018sgd}. Consequently,
\begin{equation}
\lim_{T\to\infty}\E[\mathcal R_\text{rec}^{(n),\text{FFA}}] \;=\; \min_B\ell^{(n)}(B)/(Km) \;+\; O(\eta\sigma_g^2),
\end{equation}
and the first term is fully determined by the static least-squares minimization. The $O(\eta\sigma_g^2)$ correction is a standard SGD noise-floor term, not a function of the misalignment $\delta^0(n)$, so it does not affect the qualitative shape of the floor. We absorb it into $\tilde c$ for the remainder of the proof. The function $\ell^{(n)}(B)$ is convex quadratic in $B$ with Hessian $A^0(n)X_n X_n^\top A^0(n)^\top\otimes I$, which is positive definite when $A^0(n)X_n$ has full row rank (the generic case). Setting $\nabla_B\ell^{(n)}(B)=0$ gives the normal equation
\begin{equation}
\bigl(B\,A^0(n) X_n - B^\star(n) A^\star(n) X_n\bigr)\,(A^0(n) X_n)^\top \;=\; 0,
\end{equation}
which solves to
\begin{equation}
B_\text{FFA}(n) \;=\; B^\star(n)A^\star(n)\,X_n X_n^\top A^0(n)^\top \bigl(A^0(n)\,X_n X_n^\top A^0(n)^\top\bigr)^{-1}.
\end{equation}
This is the matrix Moore--Penrose form of the least-squares solution: writing $Z=A^0(n)X_n$, the optimal $B$ is precisely $B^\star A^\star X_n\cdot Z^\dagger$. Recall $Z^\dagger=Z^\top(ZZ^\top)^{-1}$ for row full-rank $Z$. Substituting back,
\begin{equation}
\min_B \ell^{(n)}(B) \;=\; \tfrac12\bigl\|B^\star(n)A^\star(n)\,X_n\,(I_m-P^\text{row}_{A^0(n) X_n})\bigr\|_F^2,
\end{equation}
where
\begin{equation}\label{row_projector}
P^\text{row}_{A^0(n)X_n} \;:=\; X_n^\top A^0(n)^\top \bigl(A^0(n)\,X_n X_n^\top A^0(n)^\top\bigr)^{-1} A^0(n)\,X_n \;\in\;\R^{m\times m}
\end{equation}
is the orthogonal projector onto $\mathrm{row}(A^0(n)X_n)$ inside $\R^m$ (the sample-index space), confirming that $B^\star(n)A^\star(n)\,X_n\,(I_m-P^\text{row}_{A^0(n) X_n})$ is the part of $B^\star A^\star X_n$ orthogonal to $\mathrm{row}(A^0(n)X_n)$.Bring the residual into the parameter space via the trace identity
\begin{equation} \label{residual}
\bigl\|B^\star A^\star X_n(I-P^\text{row}_{A^0(n)X_n})\bigr\|_F^2 =\tr\bigl((B^\star A^\star)\,M_n\,(B^\star A^\star)^\top\bigr),\quad M_n := X_n(I-P^\text{row}_{A^0(n)X_n})X_n^\top.
\end{equation} 
Setting $S := X_n X_n^\top$ and expanding $P^\text{row}_{A^0(n)X_n}$ explicitly, $M_n$ admits the \emph{Schur complement representation}
\begin{equation} \label{M_n}
M_n \;=\; S - S\,A^0(n)^\top\bigl(A^0(n)\,S\,A^0(n)^\top\bigr)^{-1}\,A^0(n)\,S.
\end{equation}

We assume the per-layer input distribution $\mathcal D_X^{(n)}$ produces zero-mean random vectors $x\sim\mathcal D_X^{(n)}$ satisfying the shape-aware sub-Gaussian condition \eqref{cond_shapeaware} in Lemma \ref{lem:cov-conc} with parameter $\kappa^{(n)}=O(1)$ uniformly across layers and clients, which is the standard condition in or sample-covariance concentration in operator norm \cite{vershynin2018}. In specific, it is satisfied by Gaussian and log-concave distributions \cite{vershynin2018}. Throughout we treat $\kappa^{(n)}$, $\|\Sigma_X^{(n)}\|_\text{op}$, and $\|B^\star A^\star\|_\text{op}$ as $\Theta(1)$ constants in the rate notation.

Lemma~\ref{lem:cov-conc} gives $S/(Km) = \Sigma_X^{(n)} + O_p(1/\sqrt{Km})$, which under the well-conditioned input regime propagates through the Schur complement to
\begin{equation} \label{M_n_2}
M_n/(Km) =\Sigma_X^\perp \;+\; O_p(1/\sqrt{Km}), \qquad \Sigma_X^\perp :=\Sigma_X^{(n)} - \Sigma_X^{(n)} A^0(n)^\top\bigl(A^0(n)\Sigma_X^{(n)} A^0(n)^\top\bigr)^{-1} A^0(n)\,\Sigma_X^{(n)}.
\end{equation}
Here, the Schur complement $\Sigma_X^\perp\in\R^{d\times d}$ has kernel $\mathrm{row}(A^0(n))$, which can be easily derived. We have 
\begin{equation} \label{correction}
\tfrac{1}{Km}\,\|B^\star A^\star X_n(I-P^\text{row}_{A^0(n)X_n})\|_F^2 \;=\; \tr\bigl((B^\star A^\star)\,\Sigma_X^\perp\,(B^\star A^\star)^\top\bigr) \;+\; O_p(1/\sqrt{Km}),
\end{equation}
by combining \eqref{residual}--\eqref{M_n_2}.
Because $\Sigma_X^\perp$ has kernel $\mathrm{row}(A^0(n))$, only the part of $B^\star A^\star$ outside $\mathrm{row}(A^0(n))$ contributes, \emph{i.e.} $B^\star A^\star(I-P_{A^0(n)})$. Now, we apply Lemma~\ref{lem:norm-sandwich} to
\begin{equation}
\tr\bigl((B^\star A^\star)\,\Sigma_X^\perp\,(B^\star A^\star)^\top\bigr)= \|B^\star A^\star (\Sigma_X^\perp)^{1/2}\|_F^2=
\|B^\star A^\star(I-P_{A^0(n)}) (\Sigma_X^\perp)^{1/2}\|_F^2
\end{equation}
with $M= B^\star A^\star(I-P_{A^0(n)})$ and $X = (\Sigma_X^\perp)^{1/2}$ and obtain
\begin{equation}
\lambda_{\min}^+(\Sigma_X^\perp)\,(\delta^0(n))^2 \;\le\; \tr\bigl((B^\star A^\star)\,\Sigma_X^\perp\,(B^\star A^\star)^\top\bigr) \;\le\; \lambda_{\max}(\Sigma_X^\perp)\,(\delta^0(n))^2,
\end{equation}
where $\lambda_{\min}^+(\Sigma_X^\perp) > 0$ is the smallest nonzero eigenvalue (acting on $\ker A^0(n)$) and $\lambda_{\max}(\Sigma_X^\perp) \le \|\Sigma_X^{(n)}\|_\text{op} = O(1)$. Under the well-conditioned input regime, both eigenvalue bounds are $\Theta(1)$, so the sandwich's two endpoints have the same order $\Theta((\delta^0(n))^2)$. Combined with the $O_p(1/\sqrt{Km})$ correction \eqref{correction}, we have
\begin{equation}
    \min_B \ell^{(n)}(B)\cdot\tfrac{1}{Km} \;=\; \Theta\bigl((\delta^0(n))^2\bigr).
\end{equation}
Because $\mathcal R_\text{rec}$ is layer-additive \emph{by definition}, summing the per-layer $\Theta$ statements and dividing by $N$ gives, with no further independence assumption,
\begin{equation}
\lim_{T\to\infty}\E\bigl[\mathcal R_\text{rec}(\mathbf W_T^\text{FFA})\bigr] \;=\; \Theta\!\Bigl(\tfrac{1}{N}\sum_n (\delta^0(n))^2\Bigr),    
\end{equation}
the implicit constant absorbing the per-layer conditioning $\sigma_{\max}^2(X_n)/\sigma_{\min}^2(X_n A^0(n)^\dagger)$, the correction $O_p(1/\sqrt{Km})$, and any SGD noise floor $O(\eta\sigma_g^2)$. 

\emph{RoLoRA lower bound via pigeonhole.} For any layer-tied schedule $M^t\in\{\mathbf 1_A,\mathbf 1_B\}$ with $A$-active set $\mathcal R_A:=\{t\le T:M^t=\mathbf 1_A\}$ of size $|\mathcal R_A|=\lfloor T/2\rfloor$, consider an adversarial layer configuration where the per-layer \emph{optimal $A$-active sets} $\mathcal S_1,\ldots,\mathcal S_N\subseteq[T]$ partition $[T]$ into $N$ equal parts of size $T/N$ each (assuming $N\mid T$ for clarity; the general case differs by $O(1)$). Define the \emph{mismatch set for layer $n$} as $\mathcal S_n\setminus\mathcal R_A$, which are the rounds where layer $n$'s optimal protocol calls for $A$-active but the schedule has $B$ active globally, so $A^t(n)$ remains unchanged and $\delta^t(n)$ cannot contract on these rounds.

Since $\{\mathcal S_n\}$ partition $[T]$,
\begin{equation}
\sum_{n=1}^N |\mathcal R_A\cap\mathcal S_n| \;=\; |\mathcal R_A| \;=\; T/2,
\end{equation}
so by averaging, at least one layer $n^\star$ has $|\mathcal R_A\cap\mathcal S_{n^\star}|\le T/(2N)$. The mismatch count for this layer is
\begin{equation}
|\mathcal S_{n^\star}\setminus\mathcal R_A| \;=\; |\mathcal S_{n^\star}| - |\mathcal R_A\cap\mathcal S_{n^\star}| \;\ge\; T/N - T/(2N) \;=\; T/(2N) \;=\; \Omega(T).
\end{equation}
On each mismatch round, $M^t = \mathbf 1_B$ so $A^t(n^\star)$ is unchanged; the contracting (matched) rounds for layer $n^\star$ are at most $T/(2N)$, half of what AS-LoRA-style adaptive selection could achieve. In the regime where this bounded number of contractions is insufficient to drive $\delta$ to zero (e.g., finite-$T$ analysis or contraction ratio close to $1$), $\delta^T(n^\star)\ge(1-\epsilon)\,\delta^0(n^\star)$ persists for some $\epsilon\in(0,1)$. Plugging into the per-layer $\Theta$ floor yields
\begin{equation}
\E\bigl[\mathcal R_\text{rec}(\mathbf W_T^\text{RoLoRA})\bigr] \;=\; \Omega\!\Bigl(\tfrac{1}{N}\sum_n (\delta^0(n))^2\Bigr),
\end{equation}
matching the FFA $\Theta$ floor up to the multiplicative factor $(1-\epsilon)^2\in(0,1]$. \hfill$\square$

\section{Proof of Corollary~\ref{cor:asstrict}}
\label{app:proof-asstrict}

Fix layer $n$ and consider a round $t$ at which AS-LoRA activates the down-projection block $A^{(n)}$. Recall $\delta^t(n) = \|B^\star(n)A^\star(n)(I-P_{A^t(n)})\|_F$. The per-layer reconstruction loss $\ell^{(n)}(A,B) = \tfrac12\|B^\star A^\star X_n - BA X_n\|_F^2$, viewed as a function of $A$ alone (with $B$ fixed at any value achieving the layer-$n$ optimum given $A$), is convex quadratic in $A$ with Hessian sandwich $\mu_A^{(n)} I\preceq H_A^{(n)}\preceq L_A^{(n)} I$ (Assumptions~\ref{asm:smooth}, \ref{asm:hess}). Applying Lemma~\ref{lem:contract} with step size $\eta\le 1/L_A^{(n)}$ yields
\begin{equation} \label{cor_1}
\|A^{t+1}(n)-A_\text{opt}(n)\|_F^2 \;\le\; \bigl(1-\mu_A^{(n)}/L_A^{(n)}\bigr)\,\|A^t(n)-A_\text{opt}(n)\|_F^2,
\end{equation}
where $A_\text{opt}(n)$ is the local minimizer at layer $n$ and a particular gauge representative of $A^\star(n)$. As we have seen in the proof of Theorem \ref{thm:irreducible}, the map $A\mapsto P_A=A^\dagger A$ is locally Lipschitz on the manifold of full-row-rank matrices. By Davis--Kahn--Wedin sin$\Theta$ theorem \citep{wedin1972}, for $A,\widetilde A\in\R^{r\times d}$ both row-full-rank with smallest singular values bounded below by $\sigma>0$,
\begin{equation}
\bigl\|P_A - P_{\widetilde A}\bigr\|_F \; = \; \sqrt{2}\bigl\|\mathrm{sin}\Theta \bigr\|_F \;\le\; \frac{C}{\sigma}\,\bigl\|A-\widetilde A\bigr\|_F
\end{equation}
for an absolute constant $C$. Applying with $A=A^{t+1}(n)$, $\widetilde A=A_\text{opt}(n)$, and $\sigma=\sigma_{\min}(A_\text{opt}(n))$ gives
\begin{equation}\label{cor_2}
\|P_{A^{t+1}(n)}-P_{A_\text{opt}(n)}\|_F \;\le\; \frac{C}{\sigma_{\min}(A_\text{opt}(n))}\,\|A^{t+1}(n)-A_\text{opt}(n)\|_F.
\end{equation}

Using the fact that $B^\star A^\star (I-P_{A_\text{opt}})=0$ ($\mathrm{ker}(A^{\star})=\mathrm{ker}(A_{\mathrm{opt}})$), we have $\delta^t(n)=\|B^\star A^\star(P_{A_\text{opt}}-P_{A^t})\|_F$, so
\begin{equation}
\delta^{t+1}(n) \;\le\; \|B^\star A^\star\|_\text{op}\,\|P_{A_\text{opt}}-P_{A^{t+1}}\|_F.
\end{equation}
Combining \eqref{cor_1} and \eqref{cor_2}, using $\|B^\star A^\star\|_\text{op}=\Theta(1)$, $\sigma_{\min}(A_\text{opt})=\Theta(1)$ under the assumption and discussion with Theorem \ref{thm:irreducible}, one obtains the per-layer geometric contraction
\begin{equation} \label{geometric_cont}
\delta^{t+1}(n) \;\le\; \sqrt{1-\mu_A^{(n)}/L_A^{(n)}}\cdot\delta^t(n) \;=:\; \gamma_n\cdot\delta^t(n), \qquad \gamma_n\in(0,1).
\end{equation}
On rounds where $M^t(n)=B$ instead, $A^{(n)}$ is unchanged, so $P_{A^t(n)}$ and hence $\delta^t(n)$ are unchanged. Thus $\delta^t(n)$ is monotonically non-increasing across the trajectory, and contracts strictly only on $A$-active rounds.

By the hypothesis of the corollary, there exists $c\in(0,1)$ such that $|\{t\le T:M^t(n)=A\}|\ge cT$ for every $n$ (with probability one for argmax, in expectation for softmax). Combining with \eqref{geometric_cont}, we have 
\begin{equation}
\delta^T(n) \;\le\; \gamma_n^{cT}\cdot\delta^0(n).
\end{equation}
This is the geometric-decay form, where every $A$-active round multiplies $\delta$ by at most $\gamma_n<1$, and we have at least $cT$ such rounds. Solve $\gamma_n^{cT}\delta^0(n)\le\varepsilon$ for $T$ as
\begin{equation}
T \;\ge\; \frac{\log(\delta^0(n)/\varepsilon)}{c\,\log(1/\gamma_n)}.
\end{equation}
To guarantee this for every layer $n$ simultaneously, take
\begin{equation}
T \;\ge\; \max_n\;\frac{\log(\delta^0(n)/\varepsilon)}{c\log(1/\gamma_n)} \;=\; O\!\bigl(\log(1/\varepsilon)\bigr),
\end{equation}
where the absolute constant hides $c^{-1}$ and $\max_n(\log\gamma_n^{-1})^{-1}$, both finite by the assumption of this corollary itself and Assumptions~\ref{asm:smooth}--\ref{asm:hess}.

The per-layer $\Theta$ floor derived in Theorem \ref{thm:irreducible}'s proof depends on the current iterate only through $\delta(n)$. Substituting $\delta^T(n)\le\varepsilon$ uniformly into the same sandwich (with $\varepsilon$ in place of $\delta^0(n)$):
\begin{equation}
\lim_{T\to\infty}\E\bigl[\mathcal R_\text{rec}(\mathbf W_T^{\text{AS}})\bigr] \;=\; O(\varepsilon^2),
\end{equation}
where the implicit $O(1)$ constant matches the $\Theta(1)$ constant of Theorem~\ref{thm:irreducible}'s FFA floor (same per-layer conditioning $\sigma_{\max}^2(X_n)/\sigma_{\min}^2(X_n A^0(n)^\dagger)$). Whenever $\delta^0(n_0)>\varepsilon$ for some layer $n_0$, the FFA floor of Theorem~\ref{thm:irreducible} is $\Theta(\frac{1}{N}\sum_n (\delta^0(n))^2) \ge \Theta((\delta^0(n_0))^2/N) = \Theta(\delta^0(n_0)^2)$ (treating $N$ as $\Theta(1)$ in the rate notation), so the AS-LoRA bound $O(\varepsilon^2)$ is strictly smaller in order whenever $\varepsilon \ll \delta^0(n_0)$. This proves the strict inequality. \hfill$\square$

\section{Proof of Theorem \ref{thm:conv}} 
\label{app:proof-conv}

Fix layer $n$ and block label $a\in\{A,B\}$. Write $\Delta:=\Loss(W)-\E[\Loss(W-\eta\,g_a^{(n)})]$ and let $h:=\eta\,g_a^{(n)}$ (the perturbation at the $a^{(n)}$ block). Apply the Taylor expansion of Lemma~\ref{lem:taylor} to $\Loss$ as a function of the block factor $a^{(n)}$:
\begin{equation}
\Loss(W-h)=\Loss(W)-\langle\nabla_{a^{(n)}}\Loss,h\rangle+\tfrac12 h^\top H_a^{(n)}(W) h+R, \qquad |R|\le\tfrac{L_H}{6}\|h\|^3.
\end{equation}
Substituting $h=\eta g_a^{(n)}$ and rearranging,
\begin{equation} \label{delta}
\Loss(W)-\Loss(W-\eta g_a^{(n)}) \;=\; \eta\langle\nabla_{a^{(n)}}\Loss,g_a^{(n)}\rangle - \tfrac{\eta^2}{2}(g_a^{(n)})^\top H_a^{(n)} g_a^{(n)} - R,
\end{equation}
with $|R|=O(\eta^3\|g_a^{(n)}\|^3)$ uniformly bounded under Assumption~\ref{asm:hess}.

Now take expectation over the gradient noise. Decompose $g_a^{(n)}=\nabla_{a^{(n)}}\Loss+\xi_a^{(n)}$ with $\E[\xi_a^{(n)}]=0$ and $\E\|\xi_a^{(n)}\|^2\le\sigma_g^2$ (Assumption~\ref{asm:var}). Then we have
\emph{first-order term.} as $\E[\langle\nabla_{a^{(n)}}\Loss,g_a^{(n)}\rangle] = \|\nabla_{a^{(n)}}\Loss\|^2$ by unbiasedness, and \emph{second-order term} as $\E[(g_a^{(n)})^\top H_a^{(n)} g_a^{(n)}] = (\nabla_{a^{(n)}}\Loss)^\top H_a^{(n)}\nabla_{a^{(n)}}\Loss + \tr(H_a^{(n)}\Sigma_\xi)$ where $\Sigma_\xi=\E[\xi\xi^\top]$ has $\tr\Sigma_\xi\le\sigma_g^2$ and $H_a^{(n)}\preceq\beta I$, so $\tr(H_a^{(n)}\Sigma_\xi)\le\beta\sigma_g^2$.

Plugging back to \eqref{delta}, we have 
\begin{equation}
\Delta \;=\; \eta\|\nabla_{a^{(n)}}\Loss\|^2 - \tfrac{\eta^2}{2}(\nabla_{a^{(n)}}\Loss)^\top H_a^{(n)}\nabla_{a^{(n)}}\Loss - \tfrac{\eta^2}{2}\beta\sigma_g^2 + O(\eta^3).
\end{equation}
The score satisfies $\E[S_a (n)]=\|\nabla_{a^{(n)}}\Loss\|^2-\tfrac{\eta}{2}(\nabla_{a^{(n)}}\Loss)^\top H_a^{(n)}\nabla_{a^{(n)}}\Loss-\tfrac{\eta}{2}\tr(H_a^{(n)}\Sigma_\xi)+ O(\sigma_g^2)$. Accordingly, modulo $O(\eta\sigma_g^2)$ correction which is absorbed for $\eta\le 1/\beta$,
\begin{equation}\label{eq:per-round-proved}
\Delta \;\ge\; \eta\,S_a (n) - \tfrac{\eta^2\beta}{2}\sigma_g^2.
\end{equation}
Therefore, we can say that the greedy rule $\arg\max_a S_a (n)$ thus maximizes $\Delta$ up to $O(\eta^2\sigma_g^2)$.

Applying~\eqref{eq:per-round-proved} at round $t$ with the selected block $M^t(n)$, summing across layers $n\in[N]$ and rounds $t\in[T]$, and using $\Loss(W_T)\ge\Loss^\star$, we have 
\begin{equation}\label{Thm3_0}
\Loss(W_0)-\Loss^\star \;\ge\; \E[\Loss(W_0)]-\E[\Loss(W_T)] \;\ge\; \eta\!\sum_{t=1}^T\!\sum_{n=1}^N \E\bigl[S_{M^t(n)}^t(n)\bigr] - \tfrac{\eta^2\beta NT}{2}\sigma_g^2.
\end{equation}

 Decompose AS-LoRA's chosen score against the \emph{uniform-random baseline} $\tfrac{1}{2}(S_A^t(n)+S_B^t(n))$:
\begin{equation}
S_{M^t(n)}^t(n) \;=\; \tfrac{1}{2}\bigl(S_A^t(n)+S_B^t(n)\bigr) + \Bigl(S_{M^t(n)}^t(n)-\tfrac{1}{2}\bigl(S_A^t(n)+S_B^t(n)\bigr)\Bigr).
\end{equation}
Under AS-LoRA's argmax rule, $S_{M^t(n)}^t(n)=\max_a S_a^t(n)$, so the second term has expectation exactly $\Delta_{\mathcal S}^t(n)$ (the adaptive gain, defined as the gap between $\E[\max_a S_a^t(n)]$ and the uniform-baseline expectation $\tfrac{1}{2}\E[S_A^t(n)+S_B^t(n)]$ in the main text). Hence
\begin{equation} \label{Thm3_1}
\E[S_{M^t(n)}^t(n)] \;=\; \tfrac{1}{2}\E[S_A^t(n)+S_B^t(n)] + \Delta_{\mathcal S}^t(n).
\end{equation}
To convert the score expression into a gradient-norm bound, we apply Assumption~\ref{asm:hess} to each individual block-score $S_a(n)$. Using $H_a^{(n)}\preceq L_a^{(n)} I\preceq\beta I$ and $\eta\le 1/\beta$, the curvature term satisfies $(g_a^{(n)})^\top H_a^{(n)} g_a^{(n)}\le\beta\|g_a^{(n)}\|^2$, so
\begin{equation} \label{Thm3_2}
S_a(n) \;=\; \|g_a^{(n)}\|^2-\tfrac{\eta}{2}\lambda_a^{(n)}\|g_a^{(n)}\|^2 \;=\; \bigl(1-\tfrac{\eta}{2}\lambda_a^{(n)}\bigr)\|g_a^{(n)}\|^2 \;\ge\; \rho^\star\,\|g_a^{(n)}\|^2,
\end{equation}
by definition of $\rho^\star=\inf_{t,n}(1-\tfrac{\eta}{2}\lambda_{M^t(n)}^{(n)})\ge 1/2$. Combining \eqref{Thm3_1}--\eqref{Thm3_2} and $\E\|g_a^{(n)}\|^2\ge\|\nabla_{a(n)}\Loss\|^2$ (since $\E[g_a^{(n)}]=\nabla_{a(n)}\Loss$ and variance $\ge 0$), we have 
\begin{equation}\label{Thm3_3}
\sum_n\E[S_{M^t(n)}^t(n)] \;\ge\; \tfrac{\rho^\star}{2}\sum_n\E\|\nabla_{M^t(n)}\Loss\|^2 + \sum_n\Delta_{\mathcal S}^t(n).
\end{equation}

Combining \eqref{Thm3_0}, \eqref{Thm3_3} and  $\min_{t\le T}\E\|\nabla\Loss(W_t)\|^2 \le \tfrac{1}{T}\sum_t\E\|\nabla\Loss(W_t)\|^2$, we obtain
\begin{equation}
\min_{t\le T}\E\|\nabla\Loss(W_t)\|^2 \;\le\; \frac{2(\Loss(W_0)-\Loss^\star)}{\eta\rho^\star T} \;-\; \frac{2\,\overline{\Delta_{\mathcal S}}}{\rho^\star} \;+\; \frac{\eta\beta\sigma_g^2}{\rho^\star},
\end{equation}
by dividing by $\eta\rho^\star T/2$ where $\overline{\Delta_{\mathcal S}}=\tfrac{1}{T}\sum_t\tfrac{1}{N}\sum_n\Delta_{\mathcal S}^t(n)$. This is the non-convex bound.

\emph{PL refinement to geometric rate.} Under Assumption~\ref{asm:pl} ($\|\nabla\Loss\|^2\ge 2\mu(\Loss-\Loss^\star)$), the per-round inequality \eqref{eq:per-round-proved} sharpens from a gradient-norm lower bound to a function-value recursion. Combining \eqref{eq:per-round-proved} with \eqref{Thm3_2}, we have 
\begin{equation} \label{Thm3_basic}
\E[\Loss(W_t)-\Loss^\star] \;\le\; \E[\Loss(W_{t-1})-\Loss^\star] - \tfrac{\eta\rho^\star}{2}\,\E\|\nabla\Loss(W_{t-1})\|^2 + \tfrac{\eta^2\beta\sigma_g^2}{2}.
\end{equation}

Applying PL to the negative-coefficient gradient-norm term which is $-\tfrac{\eta\rho^\star}{2}\,\E\|\nabla\Loss\|^2 \le -\eta\mu\rho^\star\,\E[\Loss - \Loss^\star]$, \eqref{Thm3_basic} can be rewritten as
\begin{equation} 
\E[\Loss(W_t)-\Loss^\star] \;\le\; (1-\eta\mu\rho^\star)\,\E[\Loss(W_{t-1})-\Loss^\star] + \tfrac{\eta^2\beta\sigma_g^2}{2}.
\end{equation}

This is a first-order recursion in $\E[\Loss(W_t)-\Loss^\star]$, contracted by $\gamma:=1-\eta\mu\rho^\star \in (0,1)$ each round (provided $\eta < 1/(\mu\rho^\star)$, automatically satisfied under $\eta\le 1/\beta$ in the typical regime $\beta \ge \mu\rho^\star$). Considering the iteration, the first-order recusrion is extended to
\begin{equation} 
\E[\Loss(W_T)-\Loss^\star] \;\le\; (1-\eta\mu\rho^\star)^T(\Loss(W_0)-\Loss^\star) + \tfrac{\eta\beta\sigma_g^2}{2\mu\rho^\star}.
\end{equation}
\hfill$\square$

\section{Proof of Theorem~\ref{thm:implicit}}
\label{app:proof-implicit}

Fix layer $n$ and let $\kappa(H_a^{(n)}):=\lambda_{\max}(H_a^{(n)})/\lambda_{\min}(H_a^{(n)})$ denote the condition number of the per-block Hessian. Under AS-LoRA's argmax rule $M^t(n) = \arg\max_{a\in\{A,B\}} S_a(n)$, the selected block is a deterministic function of the stochastic gradients $\{g_a^{(n)}\}_{a\in\{A,B\}}$. Since $g_a^{(n)}$ is random (mini-batch noise, Assumption~\ref{asm:var}), $M^t(n)$ inherits this randomness, and we define the \emph{expected Hessian at the selected mode}
\begin{equation}
\overline H_t^{(n)} \;:=\; \E\bigl[H_{M^t(n)}^{(n)}\bigr] \;=\; \sum_{a\in\{A,B\}} p^t(a,n)\,H_a^{(n)},
\end{equation}
where $p^t(a,n) := \Pr\bigl[M^t(n) = a\bigr]$ is the \emph{induced selection probability} which is the probability that argmax picks block $a$, taken over the distribution of stochastic gradients. $\overline H_t^{(n)}$ is a convex combination of $H_A^{(n)}, H_B^{(n)}$ with weights $p^t(\cdot, n)$. The score $S_a(n)=\|g_a^{(n)}\|^2-\tfrac{\eta}{2}(g_a^{(n)})^\top H_a^{(n)} g_a^{(n)}$ contains a curvature-penalty term that is large when $\lambda_a^{(n)}=(g_a^{(n)})^\top H_a^{(n)}g_a^{(n)}/\|g_a^{(n)}\|^2$ is large, i.e.\ when block $a$ has high effective curvature along the gradient direction. Hence among two blocks with comparable gradient norms, the block with smaller $\lambda_a^{(n)}$ has larger $S_a(n)$ and is therefore selected by argmax. Formally, under heterogeneity $\lambda_A^{(n)}\ne\lambda_B^{(n)}$, the argmax-induced selection probability satisfies $p^t(a^\star,n)>1/2$ where $a^\star:=\arg\min_a\lambda_a^{(n)}$, with the gap $p^t(a^\star,n)-1/2$ strictly positive and growing with the curvature heterogeneity $|\lambda_A^{(n)}-\lambda_B^{(n)}|$. Accordingly, we can say that AS-LoRA favors to select $a\in{A,B}$ with low $\lambda_{a}^{(n)}$, and also maximum eigenvalue of $H_{a}^{(n)}$ following the observation in \cite{jaiswal2025low} that demonstrates that \emph{the gradient subspace aligns with the dominant directions in the Hessian}.

For symmetric matrices the maximum eigenvalue is convex \cite{horn2012matrix}, so by Jensen's inequality
\begin{equation}
\lambda_{\max}\bigl(\overline H_t^{(n)}\bigr) \;=\; \lambda_{\max}\bigl(\E_a[H_a^{(n)}]\bigr) \;\le\; \E_a[\lambda_{\max}(H_a^{(n)})] \;=\; \sum_a p^t(a,n)\,\lambda_{\max}(H_a^{(n)}).
\end{equation}
Comparing AS-LoRA's induced distribution $p^t$ against the uniform baseline,
\begin{equation}
\sum_a p^t(a,n)\,\lambda_{\max}(H_a^{(n)}) \;\le\; \tfrac{1}{2}(\lambda_{\max}(H_A^{(n)})+\lambda_{\max}(H_B^{(n)})) = \E_a^\text{unif}[\lambda_{\max}(H_a^{(n)})],
\end{equation}
because AS-LoRA favors low maximum eigenvalue as we discussed above. Composing,
\begin{equation}
\E[\lambda_{\max}(\overline H_t^{(n),\text{AS}})] \;\le\; \E[\lambda_{\max}(\overline H_t^{(n),\text{static}})],
\end{equation}
strict whenever the per-block top-spectra differ. This is the SAM-style implicit-flatness statement of Theorem~\ref{thm:implicit}. \hfill$\square$

\section{Experimental Detail} \label{app:experimental_detail}

\subsection{Differential Privacy Implementation}

We implement differentially private training using the Opacus library. 
To track the privacy budget, we primarily adopt the PRV (Privacy Random Variable) accountant, which provides tighter estimates compared to the standard RDP accountant. 
For a subset of kernel-based experiments, we instead use the RDP accountant due to technical limitations in the PRV implementation under those settings.

We employ Poisson subsampling using the \texttt{UniformWithReplacementSampler} provided by Opacus. 
Each data point is independently sampled with probability $q = B / N_k$, where $B$ is the expected batch size and $N_k$ denotes the dataset size of client $k$. 
This sampling scheme ensures correct privacy accounting for DP-SGD.

\subsection{Model and LoRA Configuration}

For language tasks, we apply LoRA to all transformer layers, specifically to the query and value projection matrices. 
The LoRA modules are configured with a dropout rate of $0.05$ and without bias terms (i.e., \texttt{bias="none"}).

For vision tasks, LoRA is similarly applied to the query and value projections across all layers. 
In addition, the classifier head is trained jointly without LoRA adaptation.

\subsection{Tokenization and Input Processing}

For GLUE tasks, the maximum sequence length is set to $128$, while for SQuAD tasks it is set to $384$. 
For question answering, we tokenize the input using the following setting:

\begin{equation*}
\texttt{tokenizer(question, context, truncation="only\_second",}
\end{equation*}
\begin{equation*}
\texttt{max\_length=L, padding="max\_length", return\_offsets\_mapping=True)}
\end{equation*}

\subsection{Implementation Details}

All experiments are conducted in float32 precision to ensure numerical stability under differential privacy. 
We reinitialize the model, optimizer, and privacy engine at each communication round to avoid unintended state accumulation in the privacy accountant.

During training, only the selected LoRA components are optimized according to the layer-wise mode configuration, while all other parameters remain frozen.

\section{Additional Experiments}
\label{sec:Appendix D}
Unless otherwise specified, all experiments follow the settings described in Section~\ref{sec:experimental_setup} and Appendix~\ref{app:experimental_detail}.

\subsection{Effect of random projection}  \label{app:random_projection}

\begin{figure}[h]
\centering
\includegraphics[width=0.4\linewidth]{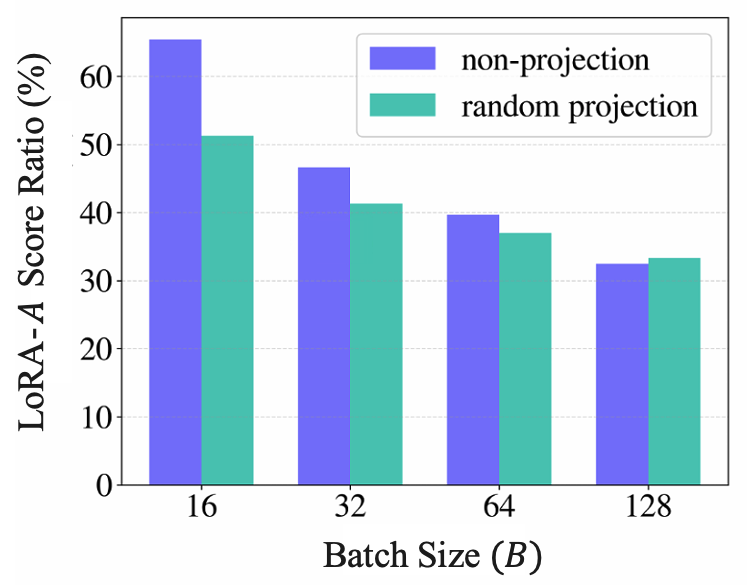}
\caption{Comparison of LoRA-A score ratio with and without random projection across different batch sizes.}
\label{fig:ablation_rp_score}
\end{figure}

To examine whether the benefit of random projection is related to channel-wise 
gradient outliers, we vary the batch size in DP-SGD, since larger batches can 
average stochastic gradient fluctuations and reduce the effect of outliers.
As shown in Table~\ref{tab:ablation_rp}, the performance gap between Proposed 
and Proposed w/o RP decreases as the batch size increases.
At $B{=}32$, random projection improves QQP accuracy by 2.03 pp, while the improvement reduces to 0.79 pp at $B{=}128$.
This trend suggests that random projection is particularly effective when smaller 
batch sizes make channel-wise outliers less likely to be canceled through batch 
averaging.
As the batch size increases and this implicit averaging effect becomes stronger, 
the additional benefit of random projection becomes smaller.
\begin{wraptable}{r}{0.45\linewidth}
\vspace{3pt}
\centering
\scriptsize
\setlength{\tabcolsep}{2pt}
\renewcommand{\arraystretch}{1.05}
\newcommand{\std}[1]{\,{\tiny$\pm$}\,\tiny #1}

\caption{Comparison with/without random projection under DP-SGD ($\epsilon=6$).}
\label{tab:ablation_rp}
\begin{tabular}{lccc}
\toprule
\textbf{Method} & QQP ($B$=32) & QQP ($B$=64) & QQP ($B$=128) \\
\midrule
Proposed
& $\textbf{63.19}$\std{0.01}
& $\textbf{73.99}$\std{1.41}
& $\textbf{79.94}$\std{0.11}\\

Proposed w/o RP
& 61.16\std{2.03}
& 70.90\std{2.78}
& 79.15\std{0.69}\\
\bottomrule
\end{tabular}

\vspace{-10pt}
\end{wraptable}

Figure~\ref{fig:ablation_rp_score} further corroborates this interpretation by 
visualizing the LoRA-$A$ score ratio across different batch sizes.
At small batch sizes (e.g., $B{=}16$), the LoRA-$A$ score ratio under 
non-projection is substantially higher than that under random projection, 
indicating that channel-wise outliers inflate the $A$ component's score and 
cause it to be disproportionately favored in mode selection.
As the batch size increases, the gap between non-projection and random projection 
narrows steadily, and at $B{=}128$ the two nearly coincide.
This confirms that larger batches implicitly suppress outlier effects through 
averaging, reducing the discrepancy between the two scoring strategies.
Taken together, these results support that the proposed random projection 
effectively mitigates channel-wise gradient outliers and leads to more balanced 
and reliable score estimation, particularly in the small-batch regime typical of 
DP-SGD training.

\begin{wraptable}{r}{0.44\linewidth}
\vspace{-13pt}
\centering
\scriptsize
\setlength{\tabcolsep}{2pt}
\renewcommand{\arraystretch}{1.05}
\newcommand{\std}[1]{\,{\tiny$\pm$}\,\tiny #1}

\caption{Comparison of single-pass and two-pass under DP-SGD ($\epsilon=3$).}
\label{tab:pass}

\begin{tabular}{lcc}
\toprule
\textbf{Method} & \textbf{MNLI (m)} & \textbf{MNLI (mm)} \\
\midrule
single-pass (ours)
& 79.15\std{0.23}
& 80.03\std{0.30} \\

two-pass inner-loop computation
& 79.37\std{0.31}
& 80.29\std{0.26} \\
\bottomrule
\end{tabular}

\vspace{-8pt}
\end{wraptable}

\subsection{Single-Pass vs. Two-Pass Score Computation}
\label{app:pass}

We compare two designs for score computation that differ in when and how the 
scores are evaluated relative to the communication rounds.

In the \emph{single-pass} design (ours), each client computes the scores as a 
by-product of local training and uploads them to the server together with the 
updated LoRA parameters.
The server then maintains EMA-smoothed scores across rounds and uses them to 
determine the mode vector for the \emph{next} round, requiring no additional 
communication beyond the standard parameter exchange.
As a result, the mode selection at round $t$ is based on historically smoothed 
scores rather than the current global model state.

In the \emph{two-pass inner-loop} design, the server first broadcasts both the 
$A$ and $B$ components of the current global model to all clients.
Each client then evaluates the scores for both components locally and sends them 
back to the server.
The server aggregates the received scores and determines the mode vector based 
on the current global model, after which the selected components are redistributed 
to the clients for local training.
This design enables greedy, round-wise mode selection based on the most 
up-to-date global model state, at the cost of an additional communication round 
and extra local computation per round.

As shown in Table~\ref{tab:pass}, the two-pass inner-loop computation achieves 
slightly higher performance than the single-pass design, improving MNLI~(m) and 
MNLI~(mm) by $+0.22$ and $+0.26$ pp, respectively.
However, the performance gap is marginal, and the single-pass design avoids the 
additional communication and computation overhead introduced by the two-pass 
approach.
These results demonstrate that the single-pass design provides a favorable 
trade-off between performance and efficiency.

\begin{wraptable}{r}{0.39\linewidth}
\vspace{-13pt}
\centering
\scriptsize
\setlength{\tabcolsep}{2pt}
\renewcommand{\arraystretch}{1.05}
\newcommand{\std}[1]{\,{\tiny$\pm$}\,\tiny #1}

\caption{Comparison of mode selection rules. Experiments are conducted under DP-SGD ($\epsilon=6$) with $B=64$.}
\label{tab:mode_selection_ablation}

\begin{tabular}{lcc}
\toprule
\textbf{Mode Selection Strategy}
& \textbf{MNLI (m)}
& \textbf{MNLI (mm)} \\
\midrule

Majority voting
& \textbf{75.56}\std{1.51}
& 76.31\std{2.38} \\

Weighted averaging
& 75.50\std{1.58}
& \textbf{76.60}\std{1.97} \\

Ours (Uniform averaging)
& 74.48\std{0.04}
& 75.15\std{0.35} \\

\bottomrule
\end{tabular}
\end{wraptable}
\subsection{Effect of aggregation rule}  \label{app:selection_rule}
We further study alternative rules for aggregating layer-wise mode decisions across clients. 
Specifically, we consider three variants: (i) \textit{majority voting}, where each client independently selects its preferred mode and the server adopts the mode receiving the majority of votes;
(ii) \textit{weighted averaging}, which computes a data-size weighted average of client scores before selecting a mode; and 
(iii) \textit{uniform averaging} (ours), which averages client scores uniformly and determines the mode accordingly.
As shown in Table~\ref{tab:mode_selection_ablation}, majority voting and weighted averaging achieve higher performance than uniform averaging on both MNLI (m) and MNLI (mm), suggesting that the aggregation rule for mode selection can meaningfully affect the quality of the selected modes. We adopt uniform averaging as the default aggregation strategy throughout this paper. Even with this simple aggregation, AS-LoRA outperforms 
prior federated LoRA methods; further gains are expected when combined 
with more sophisticated aggregation rules including majority rule, weighted aggregation.

\begin{figure*}[h]
\centering
\includegraphics[width=1.0\textwidth]{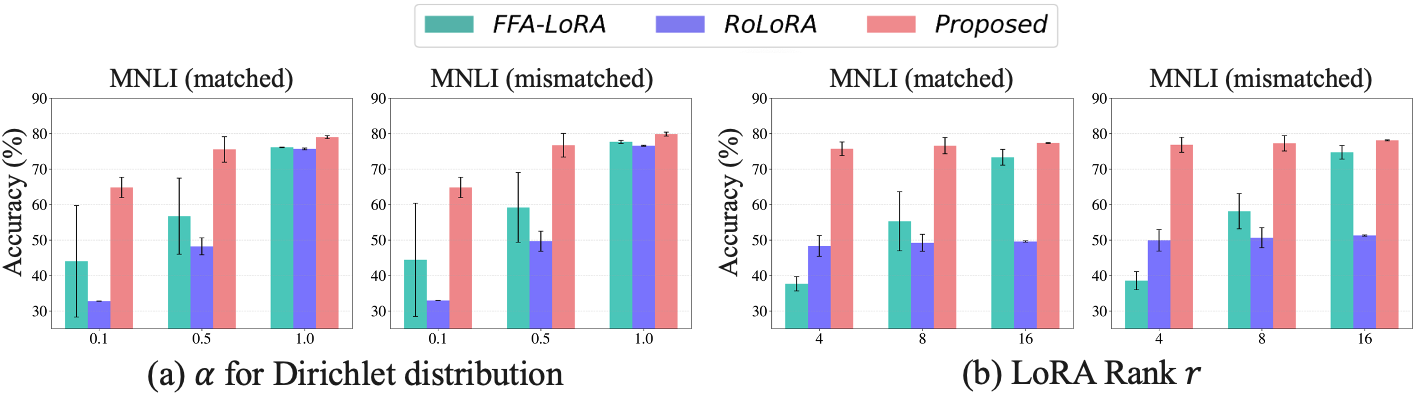}
\caption{(a): Comparison under varying dirichlet distribution $\alpha \in \{0.1,0.5,1.0\}$ on the MNLI dataset. (b): Comparison under varying LoRA rank $r \in \{4,8,16\}$ on the MNLI dataset. The number of clients is fixed at $K=3$ under DP-SGD ($\epsilon=6$).}
\label{fig:analysis_alpha_r}
\end{figure*}
\subsection{Effect of LoRA rank and data distribution} \label{app:alpha_rank}
In Figure~\ref{fig:analysis_alpha_r}(a), the proposed method consistently achieves the best performance across all $\alpha$ values and shows larger gains at smaller $\alpha$. This suggests stronger robustness under highly heterogeneous settings. The advantage of the proposed method is maintained across all heterogeneity levels.
In Figure~\ref{fig:analysis_alpha_r}(b), performance generally improves as the LoRA rank increases. The proposed method consistently outperforms the baselines across all rank settings and maintains its advantage even at low ranks. Notably, the performance improvement of FFA-LoRA is more pronounced than that of RoLoRA as the rank increases, suggesting that FFA-LoRA benefits more from the increased expressive capacity provided by larger ranks.

\subsection{Communication and memory overhead analysis} \label{app:communication_overhead}
We analyze the additional overhead introduced by AS-LoRA. Let $\mathcal{M}_n$
denote the set of LoRA modules in Transformer layer $n$. We define
\[
P_A=\sum_{n=1}^{N}\sum_{m\in\mathcal{M}_n} r d_{\mathrm{in},m}(n),
\qquad
P_B=\sum_{n=1}^{N}\sum_{m\in\mathcal{M}_n} r d_{\mathrm{out},m}(n),
\]
where $P_A$ and $P_B$ denote the number of trainable parameters in the $A$ and
$B$ LoRA components, respectively. FedLoRA communicates $P_A+P_B$ parameters
per client per round, whereas FFA-LoRA communicates $P_B$, and RoLoRA
communicates either $P_A$ or $P_B$. AS-LoRA also communicates only one active
component per layer and therefore has the same order of parameter communication
as alternating methods.

The additional communication of AS-LoRA comes only from the score and mode
information. In our implementation, all LoRA modules within the same Transformer
layer share one layer-level mode decision. Thus, each client uploads two scalar
scores per Transformer layer, and the server broadcasts one binary mode
indicator per layer. The additional uplink and downlink costs are therefore
$N$ floating-point values and $N$ bits per round, respectively. For a 24-layer
backbone, this corresponds to only $24$ FP32 values, or $96$ bytes, per client
per round, which is less than $0.02\%$ of the active LoRA parameter
communication when the hidden dimension is 1024.
\subsection{Computational overhead analysis}  \label{app:computational_overhead}

We analyze the computational overhead of the proposed curvature-based 
scoring and present practical strategies to overcome it.
We note that $\|\tilde{g}\|^2$ is obtained as a byproduct of the 
standard backward pass and incurs no additional computation.

\textbf{FLOPs accounting.}
Following~\cite{kaplan2020scaling}, we set $C_{\mathrm{bwd}} \approx 2C_{\mathrm{fwd}}$, 
so one local update step costs $3C_{\mathrm{fwd}}$.
Over $T=100$ rounds with $\tau=10$ local steps, the baseline FLOPs are:
\[
\mathcal{F}_{\mathrm{base}} = 3\tau T = 3 \times 10 \times 100 = 3{,}000\;C_{\mathrm{fwd}}.
\]
The proposed method uses a central-difference FD scheme requiring 
two forward passes per scoring target per round.
With $N_{\mathrm{target}} = 24$ targets (one active component per layer), 
the full per-round curvature scoring costs:
\[
\mathcal{F}_{\mathrm{full}} = N_{\mathrm{target}} \times 2 \times T
= 24 \times 2 \times 100 = 4{,}800\;C_{\mathrm{fwd}},
\]
yielding an overhead of $4{,}800 / 3{,}000 = \mathbf{160\%}$ over the baseline,
which is prohibitive for resource-constrained deployments.

We propose two complementary strategies to reduce this overhead.
First, replacing the central-difference scheme with a one-sided scheme:
\[
c_{a,k}^t(n) \approx 
\frac{\mathcal{L}_k^t(W_t + \epsilon_{\mathrm{fd}}\,v_{a,k}^t(n)) 
- \mathcal{L}_k^t(W_t)}{\epsilon_{\mathrm{fd}}^2},
\]
halves the per-target cost by reusing $\mathcal{L}_k^t(W_t)$ 
already computed during training.
Second, applying curvature scoring only every $f$ rounds reduces 
the number of scoring events from $T$ to $T/f$.
Under both strategies with $f=10$, the curvature overhead becomes:
\[
\mathcal{F}_{\mathrm{curv}} = N_{\mathrm{target}} \times 1 \times \frac{T}{f}
= 24 \times 10 = 240\;C_{\mathrm{fwd}},
\]
reducing the overhead to $240 / 3{,}000 \approx \mathbf{8\%}$---a 
$20\times$ reduction from the original $160\%$.

\begin{wraptable}{r}{0.34\textwidth}
\vspace{-13pt}
\centering
\scriptsize
\setlength{\tabcolsep}{2pt}
\renewcommand{\arraystretch}{1.05}
\newcommand{\std}[1]{\,{\tiny$\pm$}\,\tiny #1}
\caption{
    Comparison of curvature computation schedules in AS-LoRA under DP-SGD ($\epsilon=3$).
}
\label{app:scheduling}
\begin{tabular}{lcc}
\toprule
\textbf{Method} & \textbf{MNLI (m)} & \textbf{MNLI (mm)} \\
\midrule
\multicolumn{3}{l}{\textit{Periodic computation}} \\
($f{=}5$)   & 77.99\std{1.17} & 79.39\std{1.07} \\
($f{=}10$)  & 78.83\std{1.08} & 79.76\std{0.80} \\
\midrule
\multicolumn{3}{l}{\textit{Late-phase computation}} \\
($T^*{=}80$) & 78.91\std{1.11} & 79.87\std{0.95} \\
($T^*{=}90$) & 78.85\std{0.88} & \textbf{80.11}\std{0.74} \\
\midrule
\multicolumn{3}{l}{\textit{Original}} \\
Proposed (ours) & \textbf{79.15}\std{0.23} & 80.03\std{0.40} \\
\bottomrule
\end{tabular}
\vspace{-3pt}
\end{wraptable}

\textbf{Scheduling strategies.}
We consider two scheduling variants:
\textbf{(i) Periodic}: compute curvature scores every $f \in \{5, 10\}$ 
rounds and use $\|\tilde{g}\|^2$ otherwise.
\textbf{(ii) Late-phase}: apply $\|\tilde{g}\|^2$ for the first 
$T^* \in \{80, 90\}$ rounds and switch to curvature scoring thereafter, 
motivated by the observation that $\|\tilde{g}\|^2$ dominates early 
optimization while $g_a^\top H_a g_a$ becomes informative as 
$\|\tilde{g}\|^2$ diminishes near convergence.

\textbf{Analysis.}
As shown in Table~\ref{app:scheduling}, late-phase scheduling 
($T^*{=}80$) incurs only $-0.24$ degradation on MNLI (m) 
relative to the full method.
Notably, $T^*{=}90$ surpasses the full method on MNLI (mm) 
while reducing overhead from $160\%$ to $8\%$ ($20\times$ reduction).
These results confirm that full per-round curvature scoring is unnecessary: 
selective scheduling preserves most of the performance gain 
at a fraction of the computational cost.

\begin{wraptable}{r}{0.39\linewidth}
\vspace{-13pt}
\centering
\scriptsize
\setlength{\tabcolsep}{2pt}
\renewcommand{\arraystretch}{1.05}
\newcommand{\std}[1]{\,{\tiny$\pm$}\,\tiny #1}

\caption{Effect of gradient smoothing methods under DP-SGD ($\epsilon=3$).}
\label{tab:smoothing}

\begin{tabular}{lcc}
\toprule
\textbf{Method} & \textbf{MNLI (m)} & \textbf{MNLI (mm)} \\
\midrule
Proposed (w/o smoothing)
& 79.15\std{0.23}
& 80.03\std{0.40} \\

Gaussian
& 78.85\std{0.14}
& 79.74\std{0.21} \\

Laplacian
& \textbf{79.48}\std{0.16}
& \textbf{80.31}\std{0.16} \\

EMA
& 77.70\std{0.19}
& 78.65\std{0.14} \\
\bottomrule
\end{tabular}

\vspace{-8pt}
\end{wraptable}

\subsection{Effect of gradient smoothing under DP-SGD}\label{app:kernel}
We investigate whether applying smoothing to DP-SGD gradients can further improve utility, motivated by the observation that injected Gaussian noise in DP-SGD is high-frequency in nature and may be partially mitigated by frequency-selective filtering. Specifically, we compare three smoothing strategies applied to the aggregated gradients: (i) a standard 5-tap Gaussian filter $[1/16,\,4/16,\,6/16,\,4/16,\,1/16]$ as proposed in~\cite{liu2026rethinking}, (ii) Laplacian smoothing~\cite{liang2024differentially}, and (iii) EMA across communication rounds

The Gaussian noise injected by DP-SGD is i.i.d.\ across parameters and thus behaves as white noise in the parameter space, concentrating its energy uniformly across all frequencies.In contrast, the true gradient signal is typically smooth and sparse in the frequency domain, with energy concentrated at low frequencies.In the frequency domain, Laplacian smoothing acts as:
\begin{equation}
    \hat{g}_{\mathrm{smooth}}(\omega) 
    = \frac{\hat{g}(\omega)}{1 + \sigma \lambda(\omega)},
    \label{eq:laplacian_freq}
\end{equation}
where $\lambda(\omega)$ denotes the eigenvalue of the graph Laplacian at frequency $\omega$ and $\sigma > 0$ is a smoothing coefficient.Since $\lambda(\omega) \approx 0$ for low frequencies, the true gradient signal passes through largely intact, whereas high-frequency noise components are strongly suppressed as $\lambda(\omega)$ grows. By contrast, the 5-tap Gaussian filter is a weighted average with fixed kernel $[1/16, 4/16, 6/16, 4/16, 1/16]$, which attenuates \emph{all} frequency components including low-frequency signal, leading to signal distortion. EMA operates along the time axis rather than the parameter space, and therefore cannot exploit the spatial structure of DP noise within a single round, making it the least compatible with the spatial nature of DP-SGD noise.

Table~\ref{tab:smoothing} reports the effect of each smoothing method under $\epsilon = 3$ on the MNLI matched (m) and mismatched (mm) benchmarks. The results are consistent with our frequency-domain analysis. Laplacian smoothing is the only method that improves over the no-smoothing baseline, gaining $+0.33$ and $+0.28$ pp on MNLI (m) and MNLI (mm), respectively.
Gaussian filtering slightly degrades performance relative to the baseline ($-0.30$ pp on MNLI (m)), confirming that indiscriminate attenuation of all frequency components distorts the gradient signal. EMA yields the largest degradation ($-1.45$ pp on MNLI (m)), as temporal averaging across rounds is ill-suited for suppressing within-round spatial noise introduced by DP-SGD. These findings suggest that frequency-selective smoothing, specifically Laplacian smoothing, is the most compatible post-processing strategy for DP-SGD gradients.

\subsection{Analysis of Layer-wise Mode Selection Patterns.} \label{app:mode_selection}

Figure~\ref{fig:mode_all} visualizes the layer-wise mode selection patterns over training rounds across different tasks, revealing several non-random and task-dependent behaviors. Overall, the proposed method does not produce uniformly alternating selections across layers; instead, each task exhibits its own characteristic layer-wise preference structure, indicating that the adaptive scores capture meaningful optimization differences between LoRA-$A$ and LoRA-$B$ components rather than inducing arbitrary switching behavior.

For language understanding tasks such as QNLI, MNLI, SST-2, QQP, and SNLI, mode selection patterns often exhibit partially coherent structures across neighboring layers. Certain intermediate or upper layers consistently favor the same component over many rounds, while other layers switch more frequently. This suggests that the relative utility of LoRA-$A$ and LoRA-$B$ varies across Transformer depth, with different layers playing distinct roles in private federated adaptation depending on both task and layer.

In contrast, vision tasks exhibit more mixed and fine-grained switching. On CIFAR-100 and Tiny-ImageNet, selected modes distribute more evenly across layers and rounds, with fewer large contiguous regions dominated by a single component. This implies that in vision tasks, the relative advantage between LoRA-$A$ and LoRA-$B$ varies more locally across rounds, leading to less stable layer-wise specialization than in language tasks.

These observations have two key implications. First, they support the need for adaptive layer-wise mode selection, as a fixed global schedule cannot capture such heterogeneous preferences across tasks and layers. Second, recurring structures in language tasks suggest cross-layer dependencies not explicitly modeled by the current greedy formulation. Although each AS-LoRA layer independently selects its active component based on its score, the dynamics imply that groups of layers may prefer similar modes or evolve in a correlated manner over training rounds.

This also highlights a limitation of the current framework. Our analysis and algorithm focus on layer-wise greedy selection, enabling tractable optimization and strong empirical performance, but do not exploit inter-layer correlations. Future work could extend AS-LoRA to a joint decision framework that models structured dependencies across layers, for example by encouraging coordinated mode assignments among related layers or learning higher-level scheduling policies over layer groups. Such extensions may improve global coordination and yield more efficient update schedules beyond the current independent layer-wise strategy.

\begin{figure}[t]
    \centering

    \begin{subfigure}{0.48\linewidth}
        \centering
        \includegraphics[width=\linewidth]{./figures/8_appendix/mode_qnli.png}
        \caption{QNLI}
    \end{subfigure}
    \hfill
    \begin{subfigure}{0.48\linewidth}
        \centering
        \includegraphics[width=\linewidth]{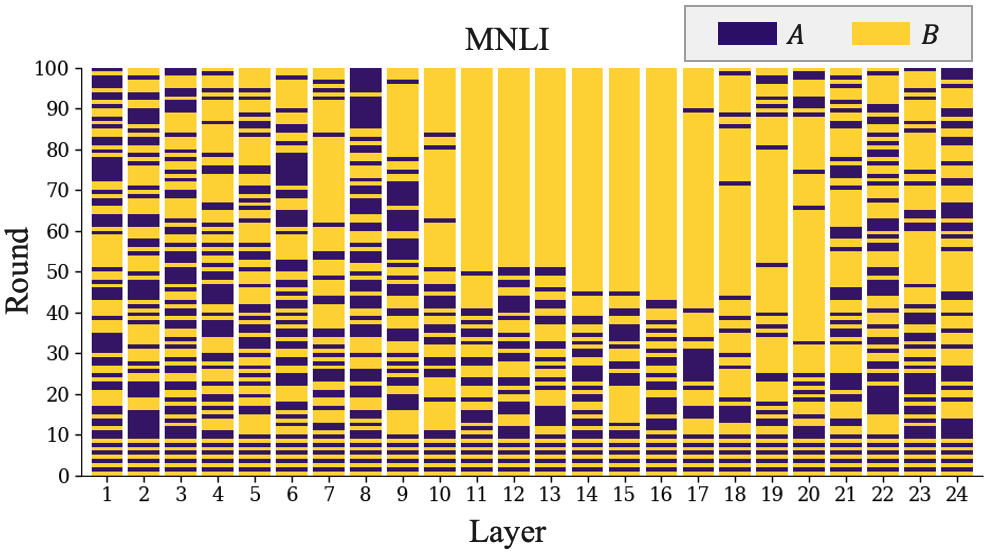}
        \caption{MNLI}
    \end{subfigure}

    \vspace{6pt}

    \begin{subfigure}{0.48\linewidth}
        \centering
        \includegraphics[width=\linewidth]{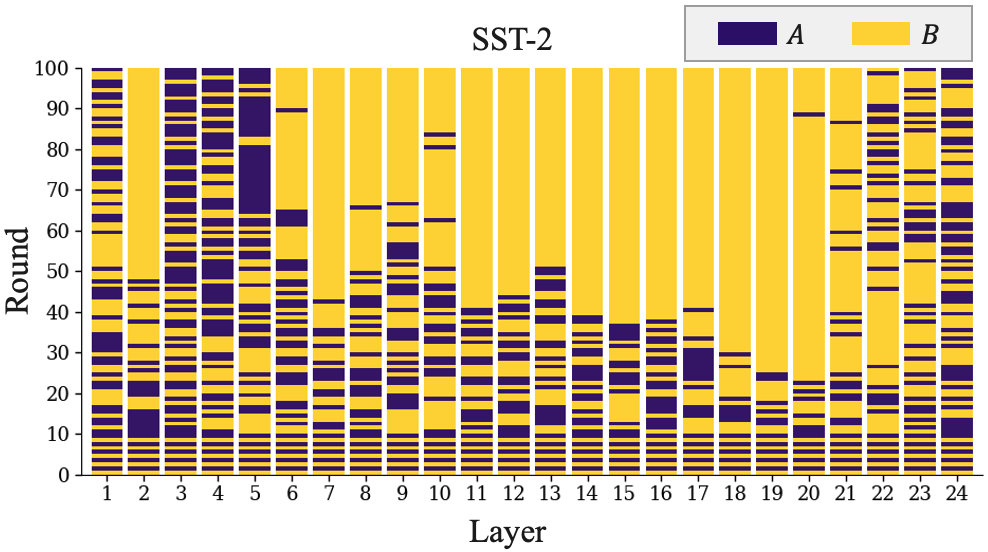}
        \caption{SST-2}
    \end{subfigure}
    \hfill
    \begin{subfigure}{0.48\linewidth}
        \centering
        \includegraphics[width=\linewidth]{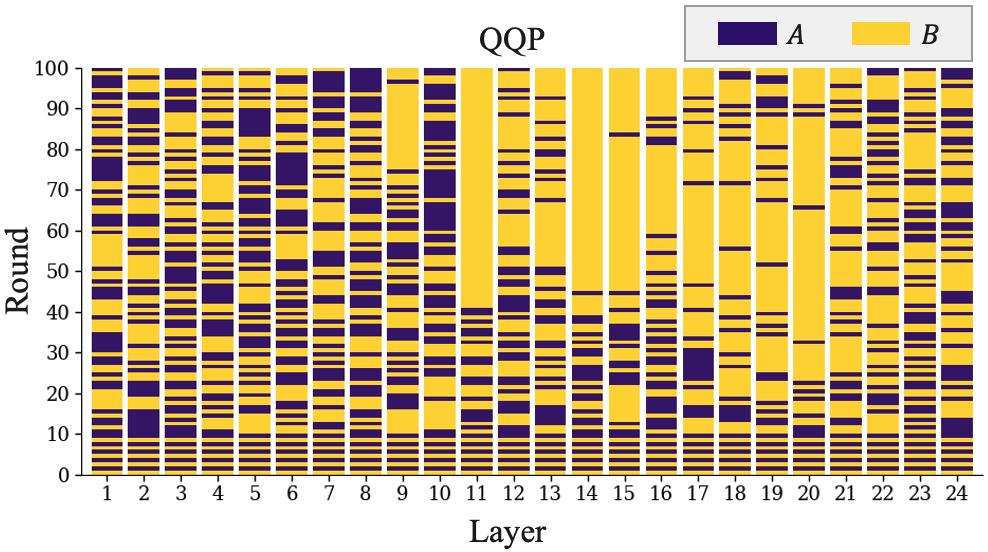}
        \caption{QQP}
    \end{subfigure}

    \vspace{6pt}

    \begin{subfigure}{0.48\linewidth}
        \centering
        \includegraphics[width=\linewidth]{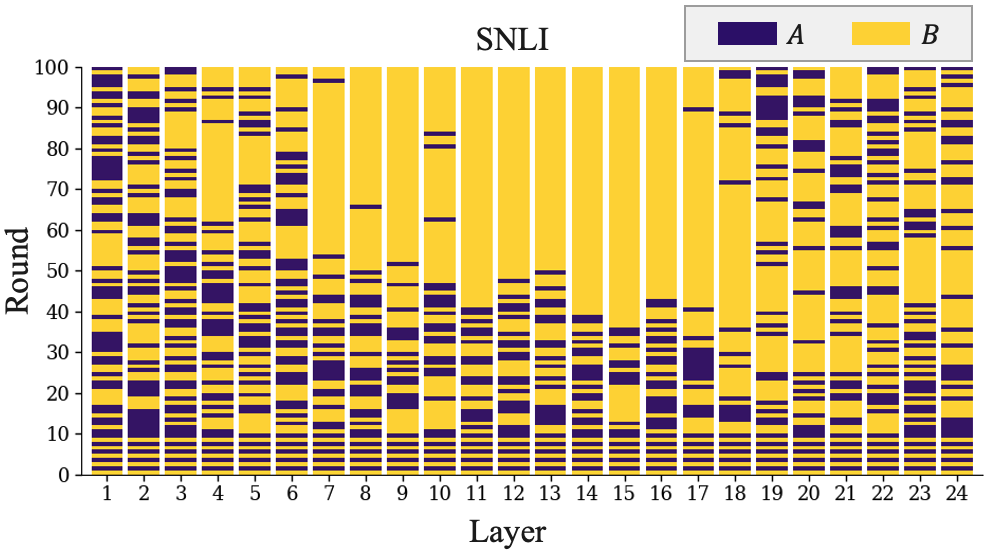}
        \caption{SNLI}
    \end{subfigure}
    \hfill
    \begin{subfigure}{0.48\linewidth}
        \centering
        \includegraphics[width=\linewidth]{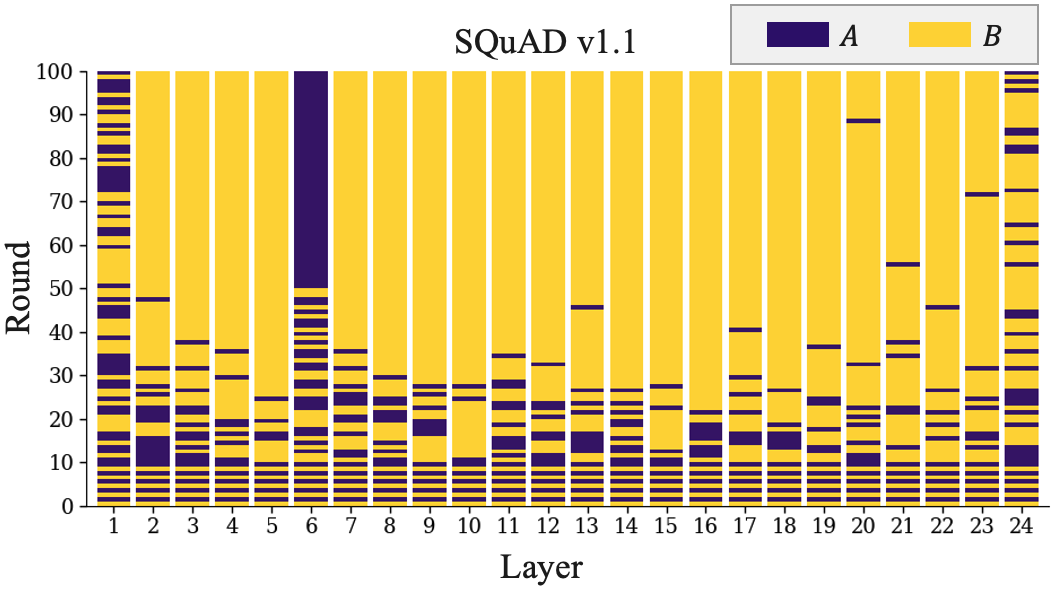}
        \caption{SQuAD v1.1}
    \end{subfigure}

    \vspace{6pt}

    \begin{subfigure}{0.48\linewidth}
        \centering
        \includegraphics[width=\linewidth]{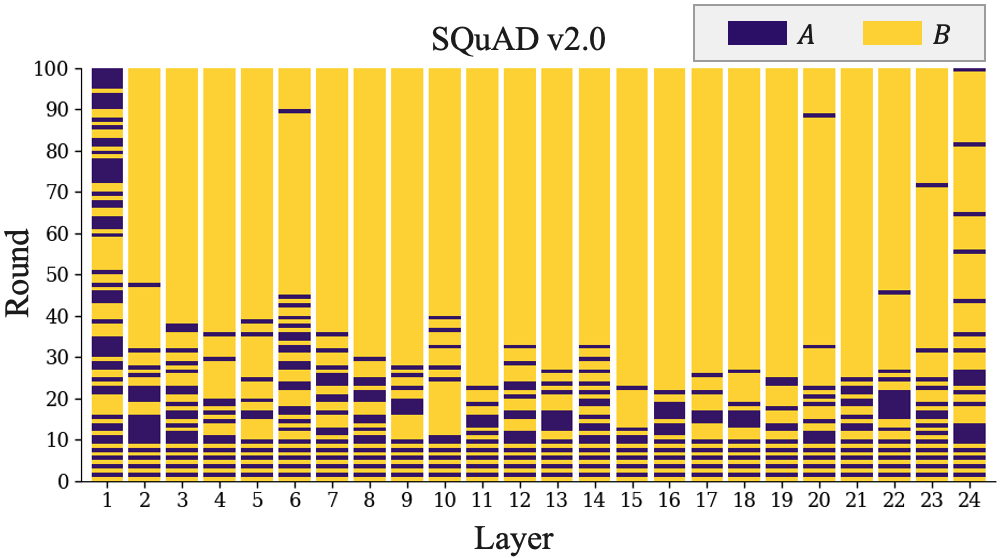}
        \caption{SQuAD v2.0}
    \end{subfigure}
    \hfill
    \begin{subfigure}{0.48\linewidth}
        \centering
        \includegraphics[width=\linewidth]{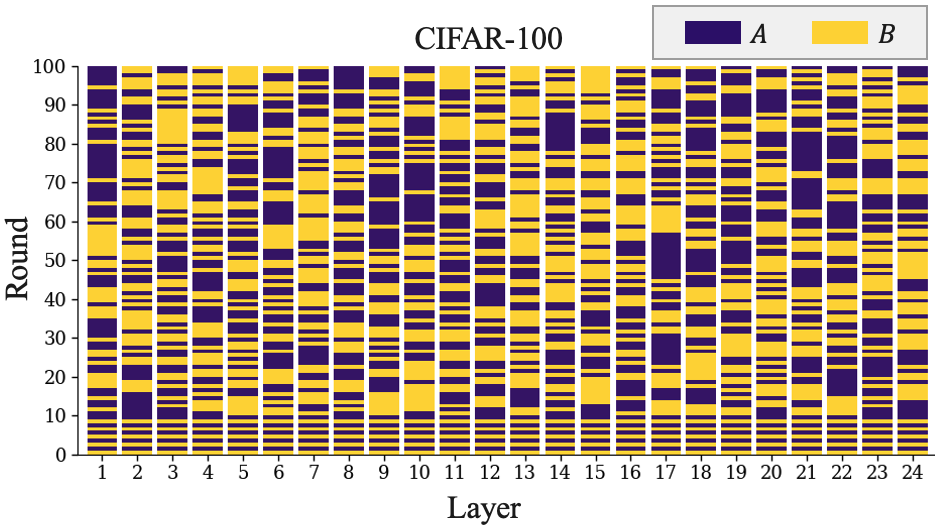}
        \caption{CIFAR-100}
    \end{subfigure}

    \vspace{6pt}

    \begin{subfigure}{0.48\linewidth}
        \centering
        \includegraphics[width=\linewidth]{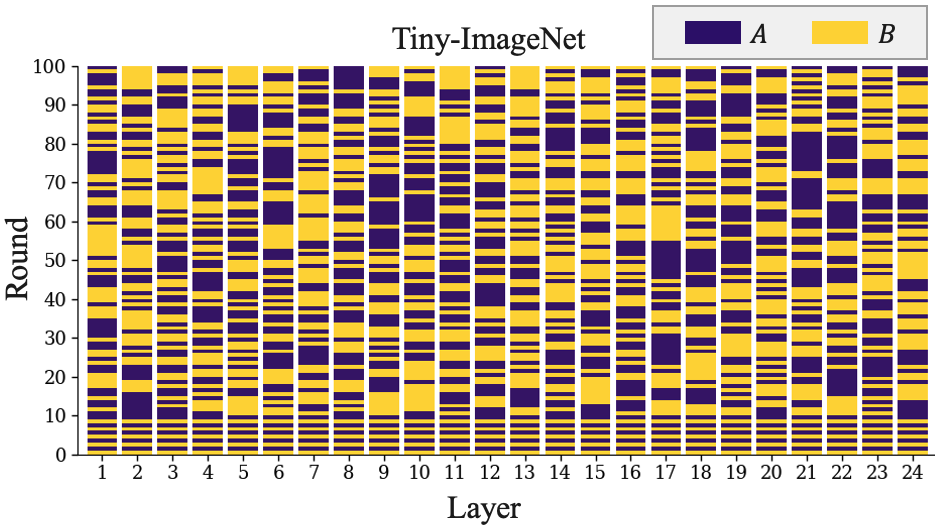}
        \caption{Tiny-ImageNet}
    \end{subfigure}

    \caption{Layer-wise mode selection patterns across different tasks.}
    \label{fig:mode_all}
\end{figure}

\end{document}